\documentclass[ejs]{imsart}

\doi{10.1214/154957804100000000}
\pubyear{0000}
\volume{0}
\firstpage{1}
\lastpage{1}

\startlocaldefs
\usepackage{amsmath}
\usepackage{graphicx}
\usepackage{enumerate}
\usepackage{url} 

\usepackage{mkolar_definitions}
\usepackage{hyperref}
\usepackage[ruled,vlined]{algorithm2e}
\endlocaldefs
\usepackage{booktabs} 

\usepackage{mathtools}
\mathtoolsset{showonlyrefs=true}
\allowdisplaybreaks

\usepackage{lipsum}
\usepackage{tikz}
\usetikzlibrary{bayesnet}
\usetikzlibrary{arrows}
\usepackage{sidecap}
\usepackage[export]{adjustbox}

\begin{document}

\begin{frontmatter}

\title{Provably Training Overparameterized Neural Network Classifiers with Non-convex Constraints}
\runtitle{Training Neural Networks with Non-convex Constraints}


\author{\fnms{You-Lin} \snm{Chen}}
\address{Department of Statistics\\
	The University of Chicago\\
	5747 S. Ellis Ave.\\
	Chicago, IL 60637 USA \\ 
	\ead[label=e1]{youlinchen@uchicago.edu}
	\printead{e1}}

\author{\fnms{Zhaoran} \snm{Wang}}
\address{Department of Industrial Engineering and Management Sciences \\
	Northwestern University \\
	2145 Sheridan Road  \\
	Evanston, IL 60208, USA \\
	\ead[label=e2]{zhaoranwang@gmail.com} \printead{e2}}

\author{\fnms{Mladen} \snm{Kolar}}
\address{The University of Chicago Booth School of Business \\
	5807 S. Woodlawn Ave. \\
	Chicago, IL 60637 USA \\
	\ead[label=e3]{mkolar@chicagobooth.edu} \printead{e3}}

\runauthor{Chen, Wang, Kolar}

\begin{abstract}
Training a classifier under non-convex constraints has gotten increasing attention in the machine learning community thanks to its wide range of applications such as algorithmic fairness and class-imbalanced classification. However, several recent works addressing non-convex constraints have only focused on simple models such as logistic regression or support vector machines. Neural networks, one of the most popular models for classification nowadays, are precluded and lack theoretical guarantees. In this work, we show that overparameterized neural networks could achieve a near-optimal and near-feasible solution of non-convex constrained optimization problems via the project stochastic gradient descent. Our key ingredient is the no-regret analysis of online learning for neural networks in the overparameterization regime, which may be of independent interest in online learning applications. 
\end{abstract}





\end{frontmatter}

\section{Introduction}~\label{sec:Introduction}

In many real-world machine learning problems, practitioners are not only interested in the performance of their models but also need to meet societal and legal goals, while taking advantage of side information, prior knowledge, and unlabeled data. For example, in the classification task, fairness metrics with respect to certain sensitive characteristics, such as gender or ethnicity \cite{chouldechova2017fair}, are measured to correct biased training data and improve the accuracy \cite{blum2019recovering}; F-measure, G-mean, and H-mean \cite{daskalaki2006evaluation, lawrence2012neural, kennedy2009learning} are used with class-imbalanced data to prevent trivial solutions; the classifier churn \cite{fard2016launch} is computed to improve prediction stability. 
All metrics mentioned above involve non-convex functions of the prediction output, and one could cast the learning task of training a classification model satisfying the above metrics as a constrained optimization problem. 
Several challenges arise due to non-convex and data-dependent constraints that
were partially addressed in the context of linear models~\cite{fard2016launch, donini2018empirical, komiyama2018nonconvex, narasimhan2019optimizing}.
Neural networks, which are widely used for classification and enjoy tremendous empirical success in practice, are not covered by the theoretical guarantees developed in these papers.

In this paper, we consider the problem of training a neural network classifier under non-convex constraints through the lens of the neural tangent kernel. Since neural network models are non-convex, we face the problem of non-convexity arising from both the constraints and the model. 
To establish the convergence rate in this challenging setting, we follow the minimax optimization strategy studied in \cite{narasimhan2019optimizing}.
Specifically, we formulate the optimization problem with non-differentiable and non-convex constraints as an unconstrained minimax optimization problem using the Lagrangian and then iteratively update model parameters, Lagrange multipliers, and auxiliary variables with convex surrogate functions for non-convex constraints. 
Although the optimization framework in \cite{narasimhan2019optimizing} is rather general, it cannot be directly applied to our problem due to the non-convexity of neural networks.

Our contribution is threefold. First, we prove the first convergence result for overparameterized neural network classifiers with non-convex constraints. In particular, our result states that a neural network classifier can achieve a near-optimal and near-feasible solution as long as the width of the neural network is large enough. Second, a no-regret guarantee is provided for online learning problems with neural networks in the overparameterization regime. The no-regret analysis may be of independent interest for online learning with neural network models and related applications. Our no-regret analysis follows by extending the conventional no-regret analysis to the online mirror descent with stochastic biased gradients. Such biased gradient analysis is studied in machine learning applications, including model-agnostic meta-learning and federated learning~\cite{denevi2019learning, t2020personalized, chen2021theorem}. Finally, unlike the prior work \cite{narasimhan2019optimizing}, our approach does not require any best response function or oracle for optimization. We emphasize that an oracle that provides a near-optimal solution for non-convex problems may not exist in practice, and the projected stochastic gradient descent is a scalable and practical algorithm for large-scale datasets and optimization problems with neural networks. As a result, our approach leads to a simple and computationally efficient procedure.

\subsection{Related Work}

Our work is related to the literature on non-convex constrained optimization, algorithmic fairness, and neural tangent kernels. We summarize the literature most related to the current work and do not attempt to provide an extensive survey.

{\bf Non-convex constrained optimization.}
Non-convex minimization problems with non-convex constraints have been recently studied due to their popularity in the machine learning community~\cite{cartis2017corrigendum, boob2019stochastic}. 
\cite{Na2021Adaptive, Na2021Inequality} developed stochastic algorithms based on sequential quadratic programming that incorporate a stochastic line search and converge globally to a first-order stationary point of the optimization program.
\cite{chen2017robust} and \cite{agarwal2018reductions} used a Lagrangian-based approach with access to an optimization oracle and found a distribution over solutions rather than a pure equilibrium. 
\cite{cotter2019two} interpreted the resulting Lagrangian as a non-zero-sum two-player game.
Their theoretical result guarantees the feasibility regarding the original constraints rather than the proxy constraints. 
Another approach to address non-convex constraints is to study weak notions of convexity \cite{davis2019stochastic}. In particular, \cite{ma2019proximally} studied constrained optimization in which both the objective and the constraints are weakly convex. They solved a sequence of strongly convex subproblems and established the computational complexity to find a nearly stationary point.

{\bf Algorithmic fairness.}
Algorithmic fairness is an important application of our non-convex constrained optimization framework. 
Many papers have studied different approaches to fulfill algorithmic fairness in machine learning~\cite{hardt2016equality, kilbertus2017avoiding, blum2019advancing, celis2019classification}. 
In particular, \cite{zafar2017fairness} included the correlation between the decision boundary and sensitive attributes as a constraint on the learned classifier. 
\cite{donini2018empirical} and \cite{oneto2019general} designed constrained optimization problems that enforce the learned classifiers to have similar errors on the positive class independently of subgroups.

{\bf Neural tangent kernel.} 
There is a considerable body of literature that analyzes deep supervised learning with overparameterized neural networks \cite{zou2018stochastic, neyshabur2018towards, li2018learning, du2019gradient}.
\cite{allen2019learning} and \cite{arora2019fine} studied two-layer neural networks in the overparametrization regime,
where neural networks can be approximated by linear models with the neural tangent kernel~\cite{jacot2018neural, lee2019wide, alemohammad2020recurrent}. 
This phenomenon of local linearization provides a powerful tool to circumvent the obstacle of the non-convexity of neural networks and establish convergence properties. 
\cite{chizat2019lazy} and \cite{allen2019convergence} showed that neural networks approximate a subset of the reproducing kernel Hilbert space induced by some kernels. 
See \cite{fan2019selective} for a recent review.
Compared to previous work, we extend the approximation results of neural networks to the setting of online learning.
Based on our new technique, we show that neural networks can be trained under non-convex constraints and the corresponding convergence rate.

\subsection{Notations}
 
We use bold capital letters ($\Ab, \Bb, \dots$) to denote matrices, bold lowercase letters ($\ab, \bb, \dots$) to denote column vectors, and lowercase or uppercase letters ($a, b, \dots A, B, \dots$) for constants. 
With slight abuse of notation, we use bold lowercase ($\fb, \gb, \dots$) and lowercase ($f, g, h, \dots$) letters for a mapping or random element whose codomain is $\RR^n$ and $\RR$, respectively, with an exception $\Lcal$ which is used for Lagrangian function. We write $\ab = (a_1, \dots, a_m)^\top$.

For a continuous function $f:\operatorname{dom} f \subset \RR^n \rightarrow \RR$, where $\operatorname{dom} f$ is the domain of $f$,  $\nabla f(\xb)$ denotes the gradient of $f$ evaluated at $\xb \in \operatorname{dom} f$. Given a vector $\xb=(x_1,  \dots, x_n)^\top \in \RR^n$, $\nabla_{x_i} f(\xb)$ denotes the partial derivative of $f$ corresponding to the variable $x_i$. We use $\|\cdot\|_2$ and $\|\cdot\|_\infty$ to denote the $\ell_2$-norm and the infinity norm, respectively. For a general norm $\|\cdot\|$, we use $\|\cdot\|_\ast$ to denote  its dual norm.

Given a set $\Ecal=\{e_1, \dots, e_k\}$ and a probability vector $\pb \in \RR^k$ such that $\sum_{i=1}^k p_i =1$, $\text{categorical}(\Ecal, \pb)$ is the categorical distribution on $\{1, \dots, k\}$ such that $\PP(e_i)=p_i$, and $\text{Uniform}(\Ecal)$ is the discrete uniform distribution on $\Ecal$ with size $k$ such that $\text{Uniform}(\Ecal) = \text{categorical}(\Ecal, (1/k, \dots, 1/k))$. We use $[n]$ to denote the set $\{1,2, \dots, n\}$.

For a subset $\Scal \subset \RR^n$, the operator $\Pi_S$ denotes the projection to $S$ with respect to the Euclidean norm $\|\cdot\|_2$. Given two sequences $a_k, b_k$, we write $a_k=O(b_k)$ if there exists a positive real number $c$ such that $a_k\leq c b_k$ for all large enough $k$.


\subsection{Organization}

The rest of the paper is organized as follows. Section~\ref{sec:Problem-Setup} sets up our problem and a minimax optimization framework as well as the details of the optimization procedure to solve our problem. Section~\ref{sec:Main-Results} discusses our main convergence results in this paper, and Section~\ref{sec:Experiments} presents numerical experiments on real-world datasets. Conclusion and the future work is provided in Section~\ref{sec:Discussion}.

\section{Problem Setup}~\label{sec:Problem-Setup}

We start by introducing a general formulation of a non-convex constrained optimization problem. Subsequently, we present few concrete examples that can be studied in the general framework. Finally, we discuss the neural network predictor as well as a minimax optimization framework to solve the non-convex constrained optimization problem.

\subsection{Optimization Formulation}

Let $\xb$ denote a feature vector in an instance space $\Xcal \subset \RR^{d}$ and 
let $z$ denote a label in a label space $\Zcal = \{-1,1\}$.
The goal is to train a model $y(\thetab; \xb)$ by finding a parameter $\thetab$ in a parameter space $\Theta$ that minimizes the following constrained optimization problem:
\begin{equation} \label{eq:optimization-2}
    \begin{split}
        \min_{\thetab \in \Theta} & \ \EE_{(\xb,z)\sim\Dcal_0} [h_0(y(\thetab;\xb), z)] \\
        \text{s.t.}
        & \ g_j \left( \EE_{(\xb,z) \sim \Dcal_1} [h_1(y(\thetab;\xb), z)], \dots, \EE_{(\xb,z) \sim \Dcal_K} [h_K(y(\thetab;\xb), z)] \right) \leq 0, \ j \in[J],
    \end{split}
\end{equation}
where $\{\Dcal_k\}_{k=0}^K$ are some distributions on $\Xcal \times \Zcal$,
$d$ is the dimension of the instance space,
$h_0$ is a convex loss function that evaluates the distance between the model $y(\thetab; \xb)$ and the label $z$ (for example, hinge loss). 
The formulation in \eqref{eq:optimization-2} is inspired by a wide range of applications including algorithmic fairness \cite{hardt2016equality, donini2018empirical}, 
class-imbalance classification \cite{kennedy2009learning, kubat1997addressing},
KL-divergence based metrics used in quantification tasks \cite{gao2015tweet, esuli2015optimizing}. 
We will discuss concrete examples of the optimization problem \eqref{eq:optimization-2} in detail in the next section.

There are three major difficulties in solving the optimization problem in~\eqref{eq:optimization-2}:
(i) non-convexity of $h_k$,
(ii) data-dependent constraints which may be computationally challenging to check,
(iii) properly sampling distributions $\Dcal_{k}$.
The third problem is beyond the scope of this paper and we present only computational issues related to formulation \eqref{eq:optimization-2}. 
We tackle the first two challenges by converting \eqref{eq:optimization-2} into a minimax optimization, and, in the corresponding Lagrangian, we replace $h_k$ by its convex upper bound, which will be discussed in detail at the end of this section. 
A similar idea is used when solving different optimization problems, such as the generalized eigenvalue problem \cite{chenonline} and the canonical correlation analysis \cite{chen2021tensor} that also involve data-dependent constraints. However, our ultimate interest is in neural networks, which require different tools from those in the existing literature.

\subsection{Concrete Examples} \label{subsec:motivation-examples}

We present practical examples that motivate our optimization problem in~\eqref{eq:optimization-2}. Note that our theory is capable but not limited to solving those examples in this section, as our non-convex constrained optimization is general enough for many other real-world applications.

{\bf Classical Classification.} The first example is the classification task aiming to minimize the following accuracy (0-1 loss) without any constraints
\begin{equation}
    \min_{\thetab} \EE_{(\xb,z)\sim \Dcal_p} \ind \{z = \sgn(y(\thetab; \xb))\},
\end{equation}
where $\Dcal_p$ is defined as the population distribution of a data set, 
$\ind$ is the indicator function, and $\sgn$ is the sign function.
Since minimizing the 0-1 loss problem is an NP-hard problem~\cite{feldman2012agnostic}, the conventional solution is to replace the 0-1 loss by a hinge loss, which is a convex upper bound for the 0-1 loss~\cite{hastie2009elements}. The same idea can be applied to constraint functions $h_k$ when considering the optimization problem in \eqref{eq:optimization-2} in its Lagrangian form.

{\bf Algorithmic Fairness.} 
In many classification problems, such as the approval of a loan or admission to a college,
the outcome $z=1$ is often treated as the {\it advantaged} outcome.
Let $\Acal \subset \Xcal$ be a protected subgroup, and let $D_\Acal$ and $D_{\Acal^c}$ denote the conditional distributions of $(\xb, z)$ given $\xb \in \Acal, z =1$ and $\xb\in \Acal^c, z =1$, respectively. With this notation, the probabilities that protected and unprotected subgroups get advantaged outcome are
\begin{equation}
    \EE_{ (\xb, z) \sim \Dcal_{\Acal}} \ind \{z = \sgn(y(\thetab; \xb))\}, \ \ \
    \EE_{ (\xb, z) \sim \Dcal_{\Acal^c}} \ind \{z = \sgn(y(\thetab; \xb))\}.
\end{equation}
In the literature on algorithmic fairness, 
the predictor $y(\thetab; \xb)$ is said to satisfy {\it equal opportunity}~\cite{hardt2016equality} if
\begin{equation} \label{eq:fairness}
    \EE_{ (\xb, z) \sim \Dcal_{\Acal}} \ind \{z = \sgn(y(\thetab; \xb))\} =
    \EE_{ (\xb, z) \sim \Dcal_{\Acal^c}} \ind \{z = \sgn(y(\thetab; \xb))\},
\end{equation}
which means that the recall on different subgroups should be aligned and the probability that the protected and unprotected subgroups get advantaged outcome should be equal. 
Note that removing sensitive features usually cannot result in equal opportunity, as there may exist other features that are highly correlated with the sensitive features.

Training a classifier that satisfies the fairness requirement stated in~\eqref{eq:fairness} can be cast as
\begin{equation} \label{eq:optimization-fairness}
    \begin{split}
        \min_{\thetab \in \Theta} & \  \EE_{(\xb,z)\sim \Dcal_p} \ind \{z = \sgn(y(\thetab; \xb))\}, \\
        \text{s.t.} & \  \EE_{ (\xb, z) \sim \Dcal_{\Acal}} \ind \{z = \sgn(y(\thetab; \xb))\} =
        \EE_{ (\xb, z) \sim \Dcal_{\Acal^c}} \ind \{z = \sgn(y(\thetab; \xb))\},
    \end{split}
\end{equation}
which can be rewritten as
\begin{equation}
    \begin{split}
        \min_{\thetab \in \Theta} & \  \EE_{(\xb,z)\sim \Dcal_p} \ind \{z = \sgn(y(\thetab; \xb))\}, \\
        \text{s.t.}
        & \ \EE_{ (\xb, z) \sim \Dcal_{\Acal}} \ind \{z = \sgn(y(\thetab; \xb))\} + \EE_{ (\xb, z) \sim \Dcal_{\Acal^c}} [-\ind \{z = \sgn(y(\thetab; \xb))\}] \leq 0, \\
        & \ \EE_{ (\xb, z) \sim \Dcal_{\Acal}} [-\ind \{z = \sgn(y(\thetab; \xb))\}] + \EE_{ (\xb, z) \sim \Dcal_{\Acal^c}} \ind \{z = \sgn(y(\thetab; \xb))\} \leq 0.
    \end{split}
\end{equation}
Therefore, training a classifier that meets the fairness requirement by minimizing 
\eqref{eq:optimization-fairness} is covered by the formulation in~\eqref{eq:optimization-2}.

{\bf Imbalance data.} 
When the data have class imbalance (for example, when $\PP(z=1)$ is small), 
finding a classifier that optimizes only the classification accuracy 
can lead to a trivial model $y(\thetab; \xb)$ that outputs only the majority class, for example, $y(\thetab; \xb) = -1$ for all $\xb$. 
Several evaluation metrics that involve precision and recall, 
such as the F-score, $2 /(\text{recall}^{-1}+\text{precision}^{-1})$, 
are remedies for naive accuracy in the class-imbalanced classification task. 
Furthermore, let $\operatorname{TPR} = \EE_{ (\xb, z) \sim \Dcal_{+}} \ind \{z = \sgn(y(\thetab; \xb))\}$ and $\operatorname{TNR} = \EE_{ (\xb, z) \sim \Dcal_{-}} \ind \{z = \sgn(y(\thetab; \xb))\}$ be the true positive and true negative rates corresponding to the predictor $y(\thetab; \xb)$, where $\Dcal_{+}$ and $\Dcal_{-}$ are the conditional distribution of $(\xb, z)$ given $z=1$ and $z=0$, respectively. 
Then, choosing $g_1(\xib) = 1-\sqrt{\xi_1 \xi_2}$ and $g_1(\xib) = 1-2/(1/\xi_1+1/\xi_2)$, we recover $\operatorname{G-mean}=1-\sqrt{\text{TPR}\times\text{TNR}}$ \cite{kubat1997addressing} and $\operatorname{H-mean}=1-2/(1/\text{TPR}+1/\text{TNR})$ \cite{kennedy2009learning}, which shows that class-imbalanced classification problems are included in the formulation~\eqref{eq:optimization}.

\subsection{Prediction Model: Two-layer Neural Network}

We focus on a two-layer neural network model 
\begin{equation} \label{eq:NN}
    y(\thetab;\xb) = \frac{1}{\sqrt{m}} \sum_{i=1}^m b_i \sigma(\ab_i^\top \xb),
\end{equation}
where $m$ is the width of the neural network, 
$b_i \in \{-1,1\}$ are the output weights, 
$\sigma(u) = \max(0,u)$ is the rectified linear unit (ReLU) activation function, 
and $\thetab=(\ab_1^\top,\dots, \ab_m^\top)^\top \in \Theta \subset \RR^{md}$ is the input weight.
Figure~\ref{fig:NN} provides a visualization of a two-layer neural network. 
As $m\rightarrow \infty$, the class of functions defined in \eqref{eq:NN} approximates a subset of the reproducing kernel Hilbert space induced by the kernel
\[
K(\xb_1,\xb_2)=\EE_{\ab \sim \Ncal(\zero,\Ib_d/d)} [\ind\{\ab^\top \xb_1>0\}\ind\{\ab^\top \xb_2>0\} \xb_1^\top \xb_2].
\]
This function class is sufficiently rich, if the width $m$ and the radius $D$ are sufficiently large \cite{arora2019fine}.  

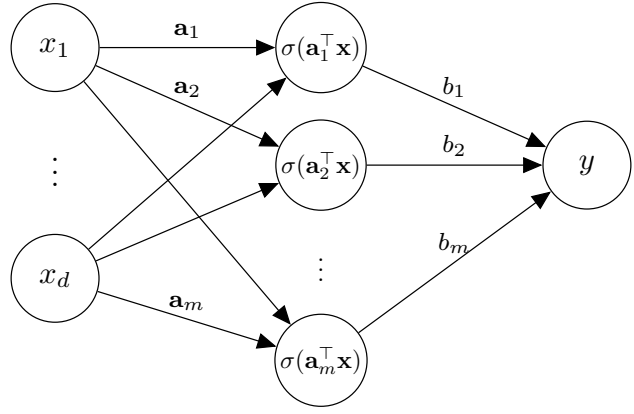
\begin{figure*}[t]
    \centering
    \resizebox{0.7\textwidth}{!}{
        \begin{tikzpicture}
            \node[latent, minimum size=1cm]                  (x1) {$x_1$};
            \node[latent, below=0.3cm of x1, minimum size=1cm, draw=none] (x2) {$\vdots$};
            \node[latent, below=0.3cm of x2, minimum size=1cm] (x3) {$x_d$};
            
            \node[latent, right=2cm of x1, scale=0.8] (ax1) {$\sigma(\ab_1^\top \xb)$};t
            \node[latent, below=0.3cm of ax1, scale=0.8] (ax2) {$\sigma(\ab_2^\top \xb)$};
            \node[latent, below=0.3cm of ax2, draw=none, scale=0.8] (ax4) {$\vdots$};
            \node[latent, below=0.3cm of ax4,scale=0.8] (ax5) {$\sigma(\ab_m^\top \xb)$};
            
            \node[latent, right=2cm of ax2, minimum size=1cm] (y) {$y$};
            
%

            \path[->,draw]
            (x1) edge node[above] {$\ab_1$} (ax1)
            (x3) edge node[above] {} (ax1)
            (x1) edge node[above] {$\ab_2$} (ax2)
            (x3) edge node[above] {} (ax2)
            (x1) edge node[above] {} (ax5)
            (x3) edge node[above] {$\ab_m$} (ax5)
            (ax1) edge node[above] {$b_1$} (y)
            (ax2) edge node[above] {$b_2$} (y)
            (ax5) edge node[above] {$b_m$} (y);
        \end{tikzpicture}
    }
    \caption{The visualization of the two-layer neural network.} \label{fig:NN}
\end{figure*}

The neural network is initialized by the following scheme
\begin{equation}
    b_i \sim \text{Uniform}(\{-1,1\}), \quad \ab_i^0 \sim \Ncal(\zero, \Ib_d/d),
\end{equation}
where $\Ncal(\zero, \Ib_d/d)$ denotes the multivariate normal distribution with mean $\zero$ and covariance matrix $\Ib_d/d$ and $\Ib_d$ is the identity matrix of size $d$. 
During training, $\thetab$ is restricted to the search space $\Thetab=\{\thetab:\| \thetab - \thetab^0\|_2 \leq D\}$, where $D$ is a predefined constant. That is,
$\thetab$ is restricted to an $\ell_2$ ball centered at initialization $\thetab^0= ((\ab_1^0)^\top, \dots, (\ab_m^0)^\top)^\top$ and $b_i$ is fixed for simplicity and omitted from $\thetab$. 
Such a setup is commonly used in the literature~\cite{allen2019What,allen2019learning, allen2019ConvergenceRate, oymak2019towards}. 
Note that $\ab_i^0 \sim \Ncal(\zero, \Ib_d/d)$, which is different from the literature that usually sets $\ab_i^0 \sim \Ncal(\zero, \Ib_d/m)$. 
In our setting, $\|\thetab^0\|$ does not converge to a fixed number in probability as $m \rightarrow \infty$. Instead, we prove that $y(\thetab;\xb)$ lies in a compact domain with high probability. This observation is crucial for the regret analysis of online learning. See Proposition~\ref{prop:y-bd}.

\subsection{A Minimax Optimization Framework and its Optimization Procedure}

We introduce a minimax optimization framework to solve the optimization problem~\eqref{eq:optimization-2}. In particular, the Lagrangian multiplier method is applied to recast the constrained optimization problem into an unconstrained minimax problem.
We introduce auxiliary variables and rewrite \eqref{eq:optimization-2} as 
\begin{equation} \label{eq:optimization}
    \begin{split}
        \min_{\thetab \in \Theta, \xib \in \Xi} & \ r_0(\thetab) \ \
        \text{s.t.}
        \ \gb \left( \xib \right) \leq \zero,\ \rb(\thetab) \leq \xib,
    \end{split}
\end{equation}
where $\gb = (g_1, \dots, g_J)^\top$, $r_k(\thetab) = \EE_{(\xb,z) \sim \Dcal_k} [h_k(y(\thetab;\xb);z)]$, $\rb=(r_1, \dots, r_K)^\top$, and $\Xi$ is the space of $\xib$, which is defined in \eqref{eq:domain} below.
Assumption~\ref{asp:1}, given in the following section, further imposes conditions that ensure that \eqref{eq:optimization} and \eqref{eq:optimization-2} are equivalent. 
The Lagrangian corresponding to \eqref{eq:optimization} is
\begin{equation} \label{eq:Lagrangian-1}
    \Lcal(\thetab, \xib, \lambdab)
    = r_0(\thetab) + \sum_{j=1}^J \lambda_j g_j(\xib) - \sum_{k=1}^K \lambda_{J+k} \xi_k
    + \sum_{k=1}^K \lambda_{J+k} r_k(\thetab),
\end{equation}
where
$\lambdab=(\lambda_1,\dots, \lambda_J, \lambda_{J+1}, \dots, \lambda_{K+J})^\top \geq \zero$ is the vector of Lagrange multipliers corresponding to the $K$ auxiliary variables and $J$ constraints.
Since the functions $h_k$ are non-convex, we assume that there exist convex surrogate functions $\tilde{h}_k$ which upper bound $h_k$, that is, 
$\tilde{h}_k$ are convex and $h_k(y, z) \leq \tilde{h}_k(y, z)$ for all $(y,z) \in \RR \times \Zcal$. 
Such convex surrogate functions generally exist. 
For example, the hinge loss is the convex surrogate function for the 1-0 loss that satisfies this requirement.
We assume that $\tilde{h}_k$ are differentiable, but it is not difficult to extend our 
argument to non-differentiable surrogate functions by using subgradient method. 
Finally, we also define $\tilde{r}_k(\thetab) = \EE_{(\xb,z) \sim \Dcal_k} [\tilde{h}_k(y(\thetab;\xb);z)]$.

A key observation is that the surrogate functions $\tilde{h}_k$ are needed only to optimize the Lagrangian with respect to $\thetab$. 
This observation plays an important role in obtaining guarantees regarding the true feasible sets. Letting
\begin{equation} \label{eq:Lagrangian-2}
\begin{split}
    \Lcal_1(\xib, \lambdab) &= \sum_{j=1}^J \lambda_j g_j(\xib) - \sum_{k=1}^K \lambda_{J+k} \xi_k, \ \
    \Lcal_2(\thetab, \lambdab) = r_0(\thetab) + \sum_{k=1}^K \lambda_{J+k} r_k(\thetab), \\
    \tilde{\Lcal}_2(\thetab, \lambdab) &= r_0(\thetab) + \sum_{k=1}^K \lambda_{J+k} \tilde{r}_k(\thetab), \ \ \tilde{\Lcal} = \Lcal_1 + \tilde{\Lcal}_2,
\end{split}
\end{equation}
we can solve \eqref{eq:optimization} by alternating the following steps: 
minimizing $\tilde{\Lcal}_2$ with respect to $\thetab$, 
minimizing $\Lcal_1$ with respect to $\xib$, and 
maximizing $\Lcal$ with respect to $\lambdab$. 
In particular, the projected stochastic gradient descent is used in each of these steps to update the parameters $\thetab$, $\xib$, and $\lambdab$. 
Given samples $\xb_{t+1}, z_{t+1} \sim \Dcal_0$ and $\xb_{t+1, k}, z_{t+1, k} \sim \Dcal_k$, 
the unbiased estimators of gradients $\nabla_{\thetab} \tilde{\Lcal}_2(\thetab,\lambdab)$ and $\nabla_{\lambdab} \Lcal(\thetab, \xib,\lambdab)$ are
\begin{equation} \label{eq:grad-est}
    \begin{split}
        \hat{\nabla}_{\thetab} \tilde{\Lcal}_2(\thetab,\lambdab)
        &= \nabla_{\thetab} h_0(y(\thetab;\xb_{t+1}), z_{t+1}) + \sum_{k=1}^{K} \lambda_{J+k} \nabla_{\thetab} \tilde{h}_k(y(\thetab;\xb_{t+1, k}),z_{t+1,k}), \\
\hat{\nabla}_{\lambdab} \Lcal(\thetab, \xib,\lambdab)
        &= 
        \begin{pmatrix}
        g_1(\xib) \\
        \vdots \\
        g_J(\xib) \\
        h_1(y(\thetab;\xb_{t+1, 1}),z_{t+1, 1}^\prime)  - \xi_{1} \\
        \vdots \\
        h_K(y(\thetab;\xb_{t+1, K}),z_{t+1, K})  -\xi_{K}        
        \end{pmatrix},
    \end{split}
\end{equation}
and the gradient $\nabla_{\xib} \Lcal_1(\xib,\lambdab)$ can be obtained as
\begin{equation} \label{eq:grad}
    \nabla_{\xib} \Lcal_1(\xib,\lambdab)
        = \left( \sum_{j=1}^J \lambda_{j} \nabla_{\xi_1} g_j(\xib) - \lambda_{J+1}, \dots, \sum_{j=1}^J \lambda_{j} \nabla_{\xi_K} g_j(\xib) - \lambda_{J+K} \right)^\top.
\end{equation}
We summarize the projected stochastic gradient descent in Algorithm~\ref{alg:surrogat-based-optimizer}.

\begin{algorithm}[!h]
    \SetAlgoLined
    \KwIn{$\thetab^0 \in \Theta, \xib^0 \in \Xi, \lambdab^0 \in \Lambda, \eta_{\theta}, \eta_{\xi}, \eta_{\lambda}$}
    \For{$t=1$ \KwTo $T$}{
        Sample $\xb_{t+1}, z_{t+1} \sim \Dcal_0$ and $\xb_{t+1, k}, z_{t+1, k}\ \sim \Dcal_k$\;
        Compute $\hat{\nabla}_{\thetab} \tilde{\Lcal}_2(\thetab^t, \lambdab^t) , \nabla_{\xib} \Lcal_1(\xib^t, \lambdab^t), \hat{\nabla}_{\lambdab} \Lcal(\thetab^t, \xib^t, \lambdab^t)$ via \eqref{eq:grad-est} and \eqref{eq:grad}\;
        $\thetab^{t+1} = \Pi_{\Theta} (\thetab^t - \eta_{\theta} \hat{\nabla}_{\thetab} \tilde{\Lcal}_2(\thetab^t,\lambdab^t))$ \;
        $\xib^{t+1} = \Pi_\Xi (\xib^t - \eta_{\xi} \nabla_{\xib} \Lcal_1(\xib^t,\lambdab^t))$ \;
        $\lambdab^{t+1} = \Pi_\Lambda (\lambdab^t + \eta_{\lambda} \hat{\nabla}_{\lambdab} \Lcal(\thetab^t, \xib^t,\lambdab^t))$ \;
    }
    \caption{Stochastic projected gradient descent} \label{alg:surrogat-based-optimizer}
\end{algorithm}

We end this section by providing an informal statement on
the convergence of the stochastic projected gradient descent
used to train the classifier parameterized by a neural network \eqref{eq:NN}.

\begin{theorem} [Informal main theorem] \label{thm:informal-main-thm}
Suppose that regularity conditions hold and that
parameters are properly tuned. Then, with high probability, we have
    \begin{equation}
        \frac{1}{T} \sum_{t=1}^T r_0(\thetab^t) - \min_{\theta \in \Theta \cap \tilde{\Fcal} } r_0(\thetab) \leq
        O\left( \frac{\kappa D^{3/2}}{\sqrt{T}} + \frac{\kappa D^{5/2}}{m^{1/4}} \right),
    \end{equation}
    and, for $j=1,\dots, J$,
    \begin{equation}
        g_j \left( \frac{1}{T} \sum_{t=1}^T \rb(\thetab^t)  \right)
        \leq O\left(\frac{ D^{3/2}}{\sqrt{T}}+ \frac{ D^2}{\kappa} + \frac{ D^{5/2}}{m^{1/4}} \right),
    \end{equation}
    where $\tilde{\Fcal} = \{\thetab : \max_j g_j(\tilde{\rb}(\theta))\leq 0\}$.
\end{theorem}

Theorem~\ref{thm:informal-main-thm} is the first result on the convergence of neural networks for non-convex constrained optimization problems.
Since Theorem~\ref{thm:informal-main-thm} provides results in terms of $\frac{1}{T} \sum r_k(\thetab^t)$,
using the convexity of $r_0$ and $g_j$, 
it can be shown that a stochastic classifier, defined in Algorithm~\ref{alg:Stochastic-classifier}, achieves a near-optimal and near feasible solution for the constrained optimization problem in \eqref{eq:optimization-2} after training the model $y(\thetab; \xb)$ for $T$ iterations.
In Algorithm~\ref{alg:Stochastic-classifier}, we select a random parameter after training the model for $T$ iterations, which is common in the literature on non-convex optimization \cite{jain2017non} and online learning \cite{shalev2012online}. 
From Theorem~\ref{thm:informal-main-thm} it follows that when the width $m$ increases, the approximation error converges to zero. 
This result is due to the fact that an infinitely wide neural network is similar to a linear model,
a phenomenon called local linearization that plays an important role in our analysis. 
We provide ingredients for the theoretical analysis and a detailed statement of the main results in the following section.

\begin{algorithm}[tb]
	\SetAlgoLined
	\KwIn{A data point $\xb$, a predictor $y(\thetab;\xb)$, $\lbrace \thetab^{t} \rbrace_{t=1}^T$ and $\pb$ such that $\sum_t p_t =1$} 
	\KwOut{$y(\thetab^i;x)$ where $i \sim \text{categorical}(T, \pb)$}
	\caption{Stochastic classifier} \label{alg:Stochastic-classifier}
\end{algorithm}

\section{Main Results}~\label{sec:Main-Results}

We provide several intermediate steps that are needed to prove Theorem~\ref{thm:informal-main-thm}. 
In particular, we discuss how to find an equilibrium by an online learning procedure with non-convex losses and local linearization of neural networks. 
For simplicity, we only consider the binary classification task and two-layer neural networks. 
However, with additional efforts, our ideas can be extended to multi-class classification and deep neural networks.

\subsection{Achieving a Near-optimal and Near-feasible Solution from Coarse-correlated Equilibrium}

We present a the convergence result for the minimax optimization framework. We need the following assumption on the objective function and constraints.
\begin{assumption} \label{asp:1}
    Let
    \begin{equation} \label{eq:domain}
        \begin{split}
            \Ycal &= \{y: |y|\leq 2D\}, \ \ \
            \Theta =\{\thetab:\|\thetab-\thetab^0 \| \leq D \}, \\
            \Xi &= \left\{\xib: |\xib| \leq \sup_{y\in \Ycal, z \in \Zcal} |\tilde{h}_k(y,z)| \right\},  \ \ \
            \Lambda = \left\{\lambdab \geq \zero:  \| \lambdab \|_{\infty} \leq \kappa \right\} ,
        \end{split}
    \end{equation}
    define the domain of every variable, where $\kappa$ and $D$ are predefined constants.
    The objective function and constraints in \eqref{eq:optimization} 
    satisfy the following restrictions:
    \begin{enumerate}
        \item $h_0$ is differentiable and convex, and $\operatorname{dom} h_k = \RR$;
        \item There exists a constant $C$ such that for all $0\leq k \leq K$
        \begin{equation}
            \sup_{y \in \Ycal, z\in \Zcal} |h_k (y, z)| \leq C;
        \end{equation}
        \item There exist functions $\tilde{h}_k$ such that each $\tilde{h}_k$ is differentiable and convex, $h_k(y, z) \leq \tilde{h}_k(y, z)$ for all $(y,z) \in \RR \times \Zcal$ and $k=1,\dots, K$, and
        \begin{equation}
            \sup_{y \in \Ycal, z\in \Zcal} |\tilde{h}_k (y, z)| \leq C;
        \end{equation}
        \item The function $g_j$ is strictly jointly convex, monotonically increasing in each argument, $\sup_{\xib \in \Xi} |g_j(\xib)|\leq C$, and $L$-Lipschitz with respect to the infinity norm, $\|\cdot\|_\infty$, for all $j=1,\dots, J$. 
    \end{enumerate}
\end{assumption}

We have several comments on Assumption~\ref{asp:1}:

(a) The boundedness assumptions for the domain of variables are crucial regular conditions to derive the local linearization in Section~\ref{subsec:NTK} and the no-regret bound for online learning in Section~\ref{subsec:finding-equilibrium-via-online-learning}. As a result, despite the fact that the behavior of local linearization around the initialization is not specific to overparameterized neural networks and may not fully explain their successes \cite{chizat2019lazy}, the boundedness assumption is natural in our setting. In other words, although the weights $\thetab$ are indeed a measure of generalization \cite{neyshabur2018pac} and regularization on weights can improve the generalization \cite{krogh1991simple, ba2016layer, salimans2016weight}, restricting $\thetab$ on $\Theta = \{\thetab:\| \thetab - \thetab^0 \|_2 \leq D \}$ and using projected gradient descent may not be practical for training neural networks. However, since projections into some compact domain are necessary for both online learning and local linearization, projected stochastic gradient descent is a natural choice for optimization in our setting.

(b) The differentiability assumptions are made for simplicity. It is not hard to extend our results to the case without differentiability via a subgradient method.

(c) Note that we do not assume $h_k$ is convex for $k=1,\dots, K$. Therefore, we need the assumption of existence of convex surrogate functions $\tilde{h}_k$, which is also a conventional solution for overcoming non-convexity in practice~\cite{hastie2009elements}. The upper boundedness is important for the guarantee of feasibility. In particular, the upper boundedness condition yields that $r_k(\thetab) \leq \tilde{r}_k(\thetab)$, so the fact that $\lambdab \geq \zero$ implies
\begin{equation} \label{eq:upper-bound}
    \Lcal_2(\thetab, \lambdab) \leq \tilde{\Lcal}_2(\thetab, \lambdab)   \Rightarrow \Lcal(\thetab, \lambdab) \leq \tilde{\Lcal}(\thetab, \lambdab).
\end{equation}

(d) The bound $\sup_{y \in \Ycal, z\in \Zcal} |\tilde{h}_k (y, z)| \leq C$ always holds due to compactness of $\Ycal$ and $\Zcal$ and continuity of $\tilde{h}_k$.

(e) The last assumption ensures that \eqref{eq:optimization} and \eqref{eq:optimization-2} are equivalent.

(f) In spite of its popularity, the F-score is not included in our framework since we cannot find a function $g$ satisfying the last part of Assumption~\ref{asp:1}.

Before discussing our theoretical results, we take algorithmic fairness in Section~\ref{subsec:motivation-examples} as an example to illustrate Assumption~\ref{asp:1}.

\begin{example}
     Consider the classification problem with fairness constraints:
    \begin{equation}
    \begin{split}
        \min_{\thetab \in \Theta} & \  \EE_{(\xb,z)\sim \Dcal_p} h_0(y(\thetab; \xb), z), \\
        \text{s.t.}
        & \ \EE_{ (\xb, z) \sim \Dcal_{\Acal}} \ind \{z = \sgn(y(\thetab; \xb))\} + \EE_{ (\xb, z) \sim \Dcal_{\Acal^c}} [-\ind \{z = \sgn(y(\thetab; \xb))\}] \leq 0, \\
        & \ \EE_{ (\xb, z) \sim \Dcal_{\Acal}} [-\ind \{z = \sgn(y(\thetab; \xb))\}] + \EE_{ (\xb, z) \sim \Dcal_{\Acal^c}} \ind \{z = \sgn(y(\thetab; \xb))\} \leq 0.
    \end{split}
\end{equation}
In this classification problem, we have
\begin{equation}
\begin{split}
    h_1(y(\thetab; \xb), z) = \ind \{z = \sgn(y(\thetab; \xb))\}, & \ \  h_2(y(\thetab; \xb), z) = -\ind \{z = \sgn(y(\thetab; \xb))\}] \\
    h_3(y(\thetab; \xb), z) = \ind \{z = \sgn(y(\thetab; \xb))\}, & \ \  h_4(y(\thetab; \xb), z) = -\ind \{z = \sgn(y(\thetab; \xb))\}] \\
    g_1(\xi_1, \xi_2, \xi_3, \xi_4) = \xi_1 + \xi_4, & \ \  g_2(\xi_1, \xi_2, \xi_3, \xi_4) = \xi_2 + \xi_3,
\end{split}
\end{equation}
and the function $h_0$ is the hinge loss $h_0(y(\thetab; \xb), z) = \max (0, 1 - z y(\thetab; \xb))$. Since the hinge loss is the convex upper bound of the zero-one loss, we can set $\tilde{h}_1$ and $\tilde{h}_3$ as the hinge loss. Similarly, we can let $\tilde{h}_2(y(\thetab; \xb), z) = \max (-1, z y(\thetab; \xb)) =  \tilde{h}_4(y(\thetab; \xb), z)$, which shows the existence of functions $\tilde{h}_k$. Note that $\ind$ and $\sgn$ are non-convex functions such that $0 \leq \ind \{z = \sgn(y(\thetab; \xb))\} \leq 1$ for all $(\xb, z) \in \Xcal \times \Zcal$. As a result, it is clear that other conditions in Assumption~\ref{asp:1} are satisfied since $\Ycal$ is compact and $g_j$ are linear. It is true that $\tilde{h}$ is not differentiable, but we can smooth the $\max$ function and make $\tilde{h}$ differentiable.
\end{example}

We are ready to present our theoretical results for the minimax optimization framework.
Under Assumption~\ref{asp:1}, Proposition~\ref{thm:equilibrium-optimization} establishes the relationship between an approximate coarse-correlated equilibrium of minimax optimization and the solution of \eqref{eq:optimization}. The proof is based on the ideas in \cite{narasimhan2019optimizing}. However, we improve these ideas by using an approximate equilibrium for $\xib$ instead of the best response or an oracle for optimization. As a result, we can use projected stochastic gradient descent to alternatively update $\xib, \thetab, \lambdab$ on large-scale data sets. Our technique is more suitable for neural network models since stochastic gradient descent is the workhorse used to train neural networks. This improvement is based on the following observation.
\begin{proposition} \label{prop:y-bd}
    Suppose that $\|\xb\|\leq 1$ for all $\xb \in \Xcal$. Then $|y(\thetab;\xb)|<2D$ with probability $1-B/D$ for some constant $B$.
\end{proposition}
\begin{proof}
    See Appendix~\ref{apd:y-bd} for a proof.
\end{proof}
Proposition~\ref{prop:y-bd} can be easily extended to an arbitrary distribution of $\{b_i\}$ using the martingale theory~\cite{chow2003probability} as long as $\{b_i\}$ are independent. We present our main result in Proposition~\ref{thm:equilibrium-optimization}.

\begin{proposition} \label{thm:equilibrium-optimization}
    Suppose Assumption~\ref{asp:1} and conditions of Proposition~\ref{prop:y-bd} hold.
    If $(\thetab^t, \xib^t , \lambdab^t)\in \Theta\times\Xi\times\Lambda$ comprises an approximate coarse-correlated equilibrium, that is, if it satisfies
    \begin{subequations} \label{eq:OL}
    \begin{align}
        \frac{1}{T} \sum_{t=1}^T \Lcal_1(\xib^t, \lambdab^t) &\leq  \min_{\xib \in \Xi} \frac{1}{T} \sum_{t=1}^T \Lcal_1(\xib,\lambdab^t) + \epsilon_\xi, \label{eq:OL-xi} \\
        \frac{1}{T} \sum_{t=1}^T \tilde{\Lcal}_2(\thetab^t, \lambdab^t) &\leq  \min_{\thetab \in \Theta} \frac{1}{T} \sum_{t=1}^T \tilde{\Lcal}_2(\thetab,\lambdab^t)  + \epsilon_\theta, \label{eq:OL-theta} \\
        \frac{1}{T} \sum_{t=1}^T {\Lcal}(\thetab^t, \xib^t, \lambdab^t) & \geq \max_{\lambdab \in \Lambda} \frac{1}{T} \sum_{t=1}^T {\Lcal}(\thetab^t,\xib^t,\lambdab) - \epsilon_\lambda, \label{eq:OL-lambda}
    \end{align}
    \end{subequations}
    where $\varepsilon_\xi, \varepsilon_\theta, \varepsilon_\lambda$ are constants that do not depend on variables $\xib, \thetab, \lambdab$, then, with probability $1-B/D$, we have
    \begin{equation} \label{eq:optimal}
        \frac{1}{T} \sum_{t=1}^T r_0(\thetab^t) - \min_{\thetab \in \Theta \cap \tilde{\Fcal} } r_0(\thetab) \leq  \epsilon_\theta + \epsilon_\xi + \epsilon_\lambda,
    \end{equation}
    and, for all $j$,
    \begin{equation} \label{eq:feasible}
        g_j \left( \frac{1}{T} \sum_{t=1}^T \rb(\thetab^t) \right) \leq (L+1)( 2C + \epsilon_\theta + \epsilon_\xi + \epsilon_\lambda) / \kappa,
    \end{equation}
    where $\tilde{\Fcal} = \{\thetab:\max_j g_j(\tilde{\rb}(\thetab))\leq 0\}$ is the feasible region and $ \kappa$ is the upper bound on $\|\lambdab\|_\infty$ in \eqref{eq:domain}.
\end{proposition}

\begin{proof}[Proof of Proposition~\ref{thm:equilibrium-optimization}]
    See Appendix~\ref{apd:equilibrium-optimization} for a proof.
\end{proof}

Proposition~\ref{thm:equilibrium-optimization} states that if there exists an approximate coarse-correlated equilibrium~\eqref{eq:OL} and if we allow for a mixed strategy or a stochastic model, we get a nearly optimal and nearly feasible solution to \eqref{eq:optimization}. In particular, if we uniformly pick an index $i \sim \operatorname{Uniform([T])}$ and use $y(\thetab^i;\xb)$ to make a prediction, then the training loss of the corresponding random classifier is $\EE_i h_0(\thetab^i) = (1/T) \sum_{t=1}^T h_0(\thetab^t)$ that is nearly optimal with the error $\epsilon_\theta + \epsilon_\xi + \epsilon_\lambda$. A similar argument holds for the feasibility. Such a stochastic model is necessary since the pure equilibrium, that is, $\thetab$ satisfying \eqref{eq:optimal} and \eqref{eq:feasible}, may not exist in general due to the non-convexity of the original problem \eqref{eq:optimization}.

We end this section with a few remarks on Proposition~\ref{thm:equilibrium-optimization}:

(a) For the optimal result in \eqref{eq:optimal}, the minimization is performed on the space $\tilde{\Fcal} = \{\thetab:\max_j g_j(\tilde{\rb}(\thetab))\leq 0\}$ instead of the space $\Fcal = \{\thetab:\max_j g_j(\rb(\thetab))\leq 0\}$. This change is caused by replacing $h_k$ with its surrogate $\tilde{h}_k$. However, note that the condition $\rb \leq \tilde{\rb}$ and the monotonicity of $g_j$ imply $\Fcal \subset \tilde{\Fcal}$ and
\begin{equation}
    \min_{\thetab \in \Theta \cap \tilde{\Fcal} } r_0(\thetab) \leq \min_{\thetab \in \Theta \cap \Fcal } r_0(\thetab).
\end{equation}

(b) The main disadvantage of Proposition~\ref{thm:equilibrium-optimization} is that the feasibility error cannot be zero unless $\kappa \rightarrow \infty$, which is a typical problem of penalized methods when used to solve constrained optimization problems \cite{bertsekas2014constrained}. Sequential quadratic programming~\cite{nocedal2006numerical} is a potential approach to eliminate the problem of infinite $\kappa$ and adaptively choose parameters.

(c) Proposition~\ref{thm:equilibrium-optimization} only discloses the possible approach to solve \eqref{eq:optimization}, but we have not discussed how to obtain an approximate coarse-correlated equilibrium. In what follows, we discuss the procedure to obtain an approximate coarse-correlated equilibrium.

\subsection{Finding Equilibrium via Online Learning} \label{subsec:finding-equilibrium-via-online-learning}

Proposition~\ref{thm:equilibrium-optimization} concludes that the coarse-correlated equilibrium with the random classifier in Algorithm~\ref{alg:Stochastic-classifier} achieves a nearly optimal and nearly feasible solution. However, it does not explain how to approach such a coarse-correlated equilibrium. In this section, we will show how online learning relates to the approximate coarse-correlated equilibrium and give details of the projected gradient descent as a concrete online learning procedure to solve online learning problems.

First, we briefly present convex online optimization problems~\cite{shalev2012online}, which are the foundation of our analysis. Given a sequence of convex functions 
\[
f_1, f_2, \dots, f_T: \Theta \rightarrow \RR,
\]
in each round $t$, the task of a learner is to choose a point $\thetab^t$ based on the information up to time $t$, and then the learner incurs a loss $f_t(\thetab^t)$. The goal of online learning is to control the learner's average regret in hindsight:
\begin{equation}\label{eq:average-regret}
    \frac{1}{T} \sum_{t=1}^T f_t(\thetab^t) - \min_{\thetab \in \Theta} \frac{1}{T} \sum_{t=1}^T f_t(\thetab).
\end{equation}
We can see that the definition of regret is equal to the definition of equilibrium in~\eqref{eq:OL}. As a result, bounding the average regret in hindsight immediately implies the approximate equilibrium.

The naive way to minimize regret is the best response strategy, that is, 
\[
\thetab^t = \arg \min_{\thetab \in \Theta} f_t(\thetab).
\]
This simple strategy leads to trivially negative average regret but may be expensive in practice. Other popular online learning algorithms include the Follow the Leader~\cite{huang2017following} and online mirror descent~\cite{srebro2011universality, shalev2012online}. We focus on the mirror descent in this work and state its update:
\begin{align*}
    \zetab_{t+1} = \nabla h^{\ast} (\nabla h(\thetab_t) - \eta \nabla f_t(\thetab^t)), \ \ \
    \thetab_{t+1} = \arg \min_{\thetab \in \Theta} B_h(\thetab, \zetab_{t+1}),
\end{align*}
where $h$ is nonnegative, differentiable, 1-strongly convex with respect to $\| \cdot \|$ with dual norm $\|\cdot\|_\ast$,
$h^{\ast}$ is the convex conjugate of $h$,
$B_h$ is the Bregman divergence with respect to $h$, defined as
\[
B_h(\thetab,\thetab^\prime) = h(\thetab^\prime) - h(\thetab) - \nabla h(\thetab^\prime)^\top(\thetab-\thetab^\prime),
\]
and $\eta$ is the stepsize. See \cite{bauschke1997legendre} for more details about the convex conjugate and the Bregman divergence. The following well-known lemma shows that online mirror descent provides a no-regret guarantee for the learner's average regret in hindsight in \eqref{eq:average-regret}.
\begin{lemma}[Lemma 2 in \cite{srebro2011universality}]
    \label{lemma:regret-bound}
    Suppose that $\Theta$ is convex,
    $\sup_{t} \|\nabla f_t(\thetab^t)\|_\ast < W$,
    and $\sup_{\thetab \in \Theta} h(\thetab) < M$. Then
    \[
    \frac{1}{T} \sum_{t=1}^T f_t(\thetab^t) - \frac{1}{T} \sum_{t=1}^T f_t(\thetab) \leq \frac{3}{2} \sqrt{\frac{W M}{T}},
    \]
    for $\eta = \sqrt{\frac{M}{WT}}$ and any $\thetab \in \Theta$.
\end{lemma}
Unfortunately,
Lemma~\ref{lemma:regret-bound} cannot be directly applied to our problem due to the non-convexity of neural networks. To analyze this more involved setting, we extend Lemma~\ref{lemma:regret-bound} to the online mirror descent with stochastic biased gradients in the next proposition. Such analysis for biased gradients is also an important technique for proving convergence of model-agnostic meta-learning and federated learning algorithms~\cite{denevi2019learning, t2020personalized, chen2021theorem}. In particular, instead of accessing true gradients $\nabla f_t(\thetab^t)$, the update rule of the stochastic mirror descent with bias is
\begin{equation} \label{eq:md}
    \begin{split}
        \zetab_{t+1} = \nabla h^{\ast} (\nabla h(\thetab^t) - \eta  \mub^t ), \ \ \
        \thetab_{t+1} = \arg \min_{\theta \in \Theta} B_h(\thetab, \zetab_{t+1}),
    \end{split}
\end{equation}
where $\mub^t$ is a biased estimate of the gradient with the bias term $\betab_t(\thetab^t)$, that is,
\[
\EE [ \mub^t \mid \thetab^t ] = \nabla f_t(\thetab^t)+\betab_t(\thetab^t).
\]

\begin{proposition}
    \label{prop:noise-regret-bound}
    Suppose that $\Theta$ is convex,
    \begin{equation}\label{eq:boundedness}
        \sup_t \|\mub^t \|_\ast < W, \ \ \ \sup_{\thetab \in \Theta} h(\thetab) < M.
    \end{equation}
    Given iterate $\thetab^t$ updated by stochastic mirror descent in \eqref{eq:md} with the bias term $\betab_t(\thetab^t)$ and $\eta = \sqrt{\frac{M}{TW}}$, we have, with probability $1-\delta$,
    \begin{multline} \label{eq:regret}
        \frac{1}{T} \sum_{t=1}^T f_t(\thetab^t) - \frac{1}{T} \sum_{t=1}^T f_t(\thetab)
        \\
        \leq \frac{3}{2} \sqrt{\frac{M W}{T}} + 8 W \sqrt{\frac{M \ln (1/\delta)}{T}}  + \frac{2\sqrt{2M}}{T} \sum_{t=1}^T \| \betab_t(\thetab^t) \|_\ast,
    \end{multline}
    for any $\thetab \in \Theta$.
\end{proposition}

\begin{proof}
    See Appendix~\ref{apd:noise-regret-bound} for a proof.
\end{proof}

To see why the extension of biased gradients is helpful, we describe a high-level idea here. Suppose that instead of observing convex losses at the time $t$, the learning task target a non-convex loss function $\bar{f}_t$. Let $\thetab^t$ be iterates of the mirror descent applied to minimize regret regarding $\bar{f}_t$. If we can find a convex approximation of $\bar{f}_t$, say $f_t$. Then, we could rewrite the regret as
\begin{multline} \label{eq:regret-example}
    \frac{1}{T} \sum_{t=1}^T \bar{f}_t(\thetab^t) - \frac{1}{T} \sum_{t=1}^T \bar{f}_t(\thetab) \\
    = \underbrace{\frac{1}{T} \sum_{t=1}^T (\bar{f}_t(\thetab^t) - f_t(\thetab^t))}_{(I)} + \underbrace{\frac{1}{T} \sum_{t=1}^T (f_t(\thetab^t) - f_t(\thetab))}_{(II)} + \underbrace{\frac{1}{T} \sum_{t=1}^T (f_t(\thetab) - \bar{f}_t(\thetab))}_{(III)},
\end{multline}
for any $\thetab$. 
The first $(I)$ and third term $(III)$ can be bounded by the approximation error between $f_t$ and $\bar{f}_t$. The convexity of $f_t$ yields the regret bound for the second term $(II)$ via Lemma~\ref{lemma:regret-bound}. However, this approach induces a bias $\betab_t(\thetab) = \nabla \bar{f}_t-\nabla f_t$ because $\thetab^t$ is updated by $\nabla \bar{f}_t$ rather than $\nabla f_t$. The key insight is that the bias $\betab_t(\thetab)$ can also be controlled, and the no-regret analysis is completed by the fact that \eqref{eq:regret-example} holds for any $\thetab \in \Theta$, including the optimum parameter in hindsight. We will apply this analysis strategy to our setting in which $\bar{f}_t$, $f_t$ are objective functions induced by a neural network and its linearization, respectively, in Section~\ref{subsec:OL-NN}. A critical tool to control the approximation of $\bar{f}_t-f_t$ and $\nabla \bar{f}_t-\nabla f_t$ is the theory of neural tangent kernel that we discuss in the following section.

\subsection{Local Linearization for Neural Networks} \label{subsec:NTK}

This section presents the phenomenon of local linearization for neural networks that requires the following regularity condition on the data distribution.
\begin{assumption}[Regularity of data distribution] \label{assumption:Regularity-Data}
	Two condition are imposed to derive local linearization for neural networks:
\begin{enumerate}
	\item For all $\xb \in \Xcal$, it holds that $\|\xb\|_2\leq 1$;
	\item For any unit vector $\eb$ and a constant $\gamma>0$, there exists $c>0$, such that $\PP_{\xb} (|\eb^\top \xb| \leq \gamma) \leq c \gamma.$
\end{enumerate}
\end{assumption}
The second regularity condition on an observation $\xb$ holds as long as the distribution of data has an upper-bounded density. Under the regularity conditions above, we can characterize that the expected approximation error of the local linearization at $\thetab^0$ vanishes to zero as $m\rightarrow \infty$. Formally speaking, define a local linearization of \eqref{eq:NN}, 
$y^0(\thetab;\xb)$, at the random initialization $\thetab^0$:
\begin{equation*}
	y^0(\thetab;\xb) 
	= \frac{1}{\sqrt{m}} \sum_{i=1}^m b_i \ind\{ (\ab_i^0)^\top \xb >0 \} \ab_i^\top \xb 
	:= \fb^0(\xb)^\top \theta,
\end{equation*}
where $\fb^0(\xb)$ is a feature map of $\xb$ such that
\[
\fb^0(\xb) = \frac{(b_1 \ind\{(\ab_1^0)^\top \xb>0\}\xb^\top, \dots, b_m \ind\{(\ab_m^0)^\top \xb>0\}\xb^\top)^\top}{\sqrt{m}}.
\]
Noting that $b_i$ and $\fb^0(\xb)$ are fixed after the initialization, it is not hard to see that $y^0$ is linear with respect to $\thetab$. The key observation is that as the width grows, neural networks exhibit similar behavior to the linear model with random features $\fb^0(x)$, which is established by the following proposition.  

\begin{proposition} \label{prop:linearization}
	Suppose that Assumption~\ref{assumption:Regularity-Data} holds.
	Then we have that for $\thetab \in \Theta = \{\thetab : \|\thetab-\thetab^0\|_2 \leq D\}$:
	\begin{enumerate}
		\item $\| \nabla y(\thetab;\xb) \|_2 \leq 1$ for $\xb \in \Xcal$;
		\item $ \EE_{\xb, \thetab^0} | y(\thetab;\xb) - y^0(\thetab;\xb) |^2 = O(D^3 m^{-1/2})$;
		\item $ \EE_{\xb, \thetab^0} \| \nabla y(\thetab;\xb) - \nabla y^0(\thetab;\xb) \|_2 = O(D m^{-1/2})$.
	\end{enumerate}
\end{proposition}

\begin{proof}
	See Appendix~\ref{apd:linearization} for a proof.
\end{proof}

We note that the result of Proposition~\ref{prop:linearization} holds uniformly over $\thetab \in \Theta$. An immediate implication is that our main result in Section~\ref{subsection:global-convergence} can be easily extended to deep learning setting since similar local linearization also holds for multi-layer neutral networks~\cite{allen2019convergence, gao2019convergence, cai2019neural}.

\subsection{Online Learning Problems for Neural Networks}
\label{subsec:OL-NN}

Since our work focuses on the setting where the prediction model $y(\thetab;\xb)$ is a neural network, $r_0(\thetab)$ and $\tilde{r}_k$ are no longer convex even though $h_0, \tilde{h}_k$ are convex. Thus, the argument used in convex online learning fails. The key insight to tackle this difficulty is that $h_0$ in \eqref{eq:optimization} is convex and if the neural network predictor $y(\thetab;\xb)$ can be well approximated by a linear model, denoted by $y^0(\thetab;\xb)$, in the overparameterization regime, then the composition of the convex loss $h_0$ and the approximated linear model $y^0$ is convex. Therefore, the theorem of convex online learning works well for the approximation part. In this section, we provide details how to modify the previous argument for online learning problem with neural networks through the lens of the neural tangent kernel. 

Consider an online learning problem under a classification setting. Assume that we have a prediction model $y(\thetab; \xb)$ for a target $z$ such that $(\xb,z) \sim \Dcal$, where $\Dcal$ is a data distribution. Instead of a fixed loss function that measures the prediction error between $y(\thetab; \xb)$ and $z$, an adversarial environment generates a new loss function $f_t$ against our prediction model $y(\thetab;\xb)$ for each $t$. Note that we do not assume the data distribution $\Dcal$ changes over time and slightly abuse the notation in the sense that $f_t$ is defined differently in the previous section. The adversarial environment is from our setting where different parameters are updated based on their own objective functions and against each other rather than the shift of the data distribution. 

At each time $t$, an adversarial environment generates a new loss functions $f_t: \RR \times \Zcal \rightarrow \RR$ such that $f_t$ is convex with respect to $y$ for all $z$. An online learning algorithm for classifier $y(\thetab;\xb)$ aims to find a sequence of parameters $\thetab^t$ based on past information that controls the regret in hindsight defined by
\begin{equation} \label{eq:regret-2}
	\frac{1}{T} \sum_{t=1}^T \EE_{\xb,z, \thetab^0} [f_t(y(\thetab^t;\xb),z)] - \min_{\thetab \in \Theta} \frac{1}{T} \sum_{t=1}^T \EE_{\xb,z, \thetab^0} [f_t(y(\thetab;\xb),z)],
\end{equation}
with high probability. Randomness is the result of the procedure to choose $\thetab^t$, which can be deterministic or stochastic. In particular, if we implement a projected gradient descent to find $\thetab^t$, then the high-probability statement can be discarded. However, since stochastic gradient descent is a standard optimization algorithm for training neural networks in practice, we chose stochastic projected gradient descent for this online learning problem. The formal definition of stochastic projected gradient descent is given as follows:
\begin{equation} \label{eq:SPGD}
	\begin{split}
		\zetab _{t+1} = \thetab^t - \eta \nabla f_t(y(\thetab^t;\xb_{t+1}), z_{t+1}), \ \ \ \thetab^{t+1} = \arg \min_{\thetab \in \Theta} \| \thetab - \zetab_{t+1} \|_2^2, 
	\end{split}
\end{equation} 
where $\{(\xb_{t},z_{t})\}_t$ is a set of independent and identical disturbed samples from~$\Dcal$.

To control regret \eqref{eq:regret-2}, we rewrite it as 
\begin{equation} \label{eq:decomposition}
	\begin{split}
		\frac{1}{T} &\sum_{t=1}^T \EE_{x,z, \thetab^0} [f_t(y(\thetab^t;\xb),z)] - \frac{1}{T} \sum_{t=1}^T \EE_{\xb,z, \thetab^0}[f_t(y(\thetab;\xb),z)] \\
		& = \underbrace{\frac{1}{T} \sum_{t=1}^T \EE_{\xb,z, \thetab^0}[f_t(y(\thetab^t;\xb),z)] - \frac{1}{T} \sum_{t=1}^T \EE_{\xb,z, \thetab^0}[f_t(y^0(\thetab^t;\xb),z)}_{(I)}] \\
		& \ \ \ +  \underbrace{\frac{1}{T} \sum_{t=1}^T \EE_{\xb,z, \thetab^0}[f_t(y^0(\thetab^t;\xb),z)] - \frac{1}{T} \sum_{t=1}^T \EE_{\xb,z, \thetab^0}[f_t(y^0(\thetab;\xb),z)]}_{(II)} \\
		& \ \ \ + \underbrace{ \frac{1}{T} \sum_{t=1}^T \EE_{\xb,z, \thetab^0}[f_t(y^0(\thetab;\xb),z)] - \frac{1}{T} \sum_{t=1}^T \EE_{\xb,z, \thetab^0}[f_t(y(\thetab;\xb),z)]}_{(III)}, \\
	\end{split}
\end{equation}
for any $\thetab \in \Theta$.
The second term $(I)$ and $(III)$ can be bounded by controlling the linearization error in Proposition~\ref{prop:linearization}. The second term $(II)$ in \eqref{eq:decomposition} can be controlled by the classical result on convex online learning, since now $y^0(\thetab;\xb)$ is linear with respect to $\thetab$ and $f_t(y^0(\thetab^t;\xb),z)$ is convex with respect to $\thetab$. However, replacing the original stochastic gradient $\nabla f_t(y(\thetab^t;\xb_{t+1}),z_{t+1})$ by the gradient of linearization $\nabla f_t(y^0(\thetab^t;\xb_{t+1}),z_{t+1})$ introduces additional error in the update. To address this, we cast \eqref{eq:SPGD} as
\begin{equation}\label{eq:PSGD}
	\begin{split}
		\zetab_{t+1} & 
		= \thetab^t - \eta \nabla  f_t(y^0(\thetab^t;\xb_{t+1}),z_{t+1}) 
		- \betab_t(\thetab^t),\\
		\thetab^{t+1} &= \arg \min_{\theta \in \Theta} \| \thetab - \zetab_{t+1} \|_2^2,
	\end{split}
\end{equation}
where 
\[
\betab_t(\thetab^t) = 
\nabla f_t(y(\thetab^t;\xb_{t+1}), z_{t+1})  -  \nabla f_t(y^0(\thetab^t;\xb_{t+1}), z_{t+1}).
\]
It turns out that the noise $\betab_{t}$ can also be controlled by
the local linearization. In summary, we obtain the following theorem for the regret bound \eqref{eq:regret-2} stating that under regular conditions for smoothness, we can bound the regret with extra error term controlled by the approximation error of local linearization.

\begin{theorem}
	\label{thm:OL-NN}
	Suppose the conditions of Proposition~\ref{prop:linearization} hold. 
	Assume that $f_t(y,z)$ is convex such that for all $y, y_1, y_2 \in \Ycal$, $t, z$
	\begin{equation} \label{eq:thm-asp}
		|\nabla_y f_t(y,z)| \leq L_f, \ \ |\nabla_y f_t(y_1,z) - \nabla_y f_t(y_2,z)| \leq L_f |y_1-y_2|.
	\end{equation}
	Then, with probability $1-\delta$, we have
	\begin{multline*}
		\frac{1}{T} \sum_{t=1}^T \EE_{x, z, \thetab^0} [f_t(y(\thetab^t;\xb),z)] - \min_\theta\frac{1}{T}  \sum_{t=1}^T  \EE_{x, z, \thetab^0} [f_t(y(\thetab;\xb),z)] \\
		\leq O \left( L_f D \sqrt{\frac{\ln(1/\delta)}{T}} + \frac{L_f^2 D^{3/2}}{m^{1/4}} \right),
	\end{multline*}
	where $\thetab^t$ is computed by \eqref{eq:PSGD} and $\eta = \sqrt{\frac{D^2}{2 L_f T}}$.
\end{theorem}

\begin{proof}[Proof of Theorem~\ref{thm:OL-NN}] 
    See Appendix~\ref{apd:OL-NN} for a proof.
\end{proof}

The smoothness condition is necessary for controlling the approximation error via local linearization. Note that when the width $m \rightarrow \infty$, the regret in hindsight is arbitrarily small as long as $T$ is large enough, which reduces to the case of classical online learning. Equipped with Theorem~\ref{thm:OL-NN} and Proposition~\ref{thm:equilibrium-optimization}, We arrive our final result in the next section.

\subsection{Global Convergence for Overparameterized Neural Netowrk} \label{subsection:global-convergence}

We prove our main theorem on global convergence regarding provable training neural network classifiers with non-convex constraints. As we can see in Theorem~\ref{thm:OL-NN} and its proof, the smoothness condition in \eqref{eq:thm-asp} is needed to apply local linearization. Thus, we make the following regularity condition.

\begin{assumption}[Regularity of Objectives and constraints] \label{assumption:Regularity-OC}
	Assume that for any $y, y^\prime \in \Ycal, z \in \Zcal$, $\xib, \xib^\prime \in \Xi$, $k=1,\dots, K$, $j=1,\dots, J$,
	\begin{enumerate}
		\item $|\nabla h_0(y,z)| \leq L$ and $|\nabla_y h_0(y,z) - \nabla_y h_0(y^\prime,z)|\leq L |y-y^\prime|$;
		
		\item $|\nabla  \tilde{h}_k(y,z)| \leq L$ and $|\nabla_y \tilde{h}_k(y,z) - \nabla_y \tilde{h}_k(y^\prime,z)|\leq L |y-y^\prime|$;
		
		\item $\|\nabla_{\xib} g_j(\xib) - \nabla_{\xib} g_j(\xib^\prime)\| \leq L \|\xib-\xib^\prime\|_\infty$;
		
		\item $\|\nabla_{\xib} g_j(\xib) \| \leq L $.
		
	\end{enumerate}
\end{assumption}

Note that Assumption~\ref{assumption:Regularity-OC} always holds thanks to the compactness of $\Ycal, \Zcal$ and the continuity of $h_0, \tilde{h}_k, g_j$ from Assumption~\ref{asp:1}. Assumption~\ref{assumption:Regularity-OC} is imposed only because it allows us to characterize the convergence rate explicitly.

Our main result shows that if the classifier is parameterized by a neural network \eqref{eq:NN}, and we run projected stochastic gradient descent in Algorithm~\ref{alg:surrogat-based-optimizer}, then we obtain a nearly optimal and nearly feasible solution of the constrained optimization problem \eqref{eq:optimization-2}. Without loss of generality, we assume $D, \kappa \geq 1$.

\begin{theorem}
\label{thm:main-thm}
Suppose that Assumption~\ref{asp:1}, \ref{assumption:Regularity-Data} and \ref{assumption:Regularity-OC} hold. Set 
\begin{equation*}
    \eta_{\theta}= \sqrt{\frac{D^2}{4 \kappa K L T}}, \ \ \  \eta_{\xi} = \sqrt{\frac{C}{2 T\kappa J L K^{1/2}}}, \ \ \ \eta_{\lambda} = \sqrt{\frac{\kappa}{2 L C  K^{1/2} J^{1/2} T}}.
\end{equation*}
Then, with probability $1-3\delta$, the iterates $\{\thetab^t, \xib^t, \lambdab^t\}_{t=1}^T$ 
of Algorithm~\ref{alg:surrogat-based-optimizer} comprise an 
approximate coarse-correlated equilibrium
\eqref{eq:OL-xi}, \eqref{eq:OL-theta}, \eqref{eq:OL-lambda} with
\begin{equation} \label{eq:equilibrium}
    \begin{split}
        \epsilon_\theta =& O\left( \kappa L D \sqrt{\frac{\ln(1/\delta)}{T}} +  \frac{\kappa (K\kappa L)^2 D^{3/2}}{m^{1/4}}\right), \\
        \epsilon_\xi =& O\left(\frac{\kappa D L^2 \sqrt{\ln(1/\delta)}}{\sqrt{T}} \right), \\ 
        \epsilon_\lambda =& O\left( \frac{D L^2 \sqrt{\kappa \ln(1/\delta)}}{\sqrt{T}} \right).
    \end{split}    
\end{equation}
Moreover, with probability $1-3\delta-B/D$, we have
\begin{equation} \label{eq:optimal-2}
    \frac{1}{T} \sum_{t=1}^T r_0(\thetab^t) - \min_{\thetab \in \Theta \cap \tilde{\Fcal} } r_0(\thetab) \leq
    O\left(  \frac{\kappa L D^{3/2} \sqrt{ \ln(1/\delta)}}{\sqrt{T}} + \frac{\kappa D^{5/2}}{m^{1/4}} \right),
\end{equation}
and, for $j=1,\dots, J$,
\begin{equation} \label{eq:feasible-2}
    g_j \left( \frac{1}{T} \sum_{t=1}^T \rb(\thetab^t)  \right)
    \leq O\left(\frac{ L D^{3/2} \sqrt{ \ln(1/\delta)}}{\sqrt{T}}+ \frac{L D^2}{\kappa} + \frac{ D^{5/2}}{m^{1/4}} \right),
\end{equation}
where $\tilde{\Fcal} = \{\thetab : \max_j g_j(\tilde{\rb}(\theta))\leq 0\}$.
\end{theorem}

\begin{proof}
    See Appendix~\ref{apd:main-thm} for a proof.
\end{proof}

To the best of our knowledge, Theorem~\ref{thm:main-thm} is the first result that shows convergence for non-convex constrained optimization problems with neural networks. Compared to~\cite{narasimhan2019optimizing}, we do not require any best-response function or oracle for optimization and obtain a similar convergence rate if ignoring the approximation error terms involving the width $m$. Most importantly, our result allows sophisticated neural networks that are non-convex and hard to analyze rather than simple linear models. We can see that as the width $m$ increases, the approximation error in \eqref{eq:optimal-2} and \eqref{eq:feasible-2} converges to zero. This result shows that an infinitely wide neural network could achieve an optimal and nearly feasible solution similar to linear models. Moreover, there is a trade-off in the error bound for optimality in \eqref{eq:optimal-2} and feasibility \eqref{eq:feasible-2}: a larger $\kappa$ to penalize the violation of constraints improves feasibility and results in a better error bound in \eqref{eq:feasible-2}, while increasing $\kappa$ hurts the convergence of Algorithm~\ref{alg:surrogat-based-optimizer}.

\section{Experiment on COMPAS}
\label{sec:Experiments}

We illustrate our algorithm and theorems by a real application in algorithmic fairness. COMPAS \cite{dressel2018accuracy} is a data set provided by ProPublicas and contains criminal history, jail and prison time, demographics, and COMPAS risk scores for Broward County defendants. The goal of our experiments is to predict recidivism under fairness constraints. That is, we would like to ensure that some protected subgroups are treated equally by prediction models. We pre-processed the data and removed observations
with some missing features and used one-hot encoding for categorical features. After removing data points for which some of the features are missing, we have 6172 samples in COMPAS. The data are then randomly divided into two groups: 70\% of the samples are for training and 30\% samples are for testing. The classifier we use here is a 2 layer neural network with $m=201$ hidden units. We restrict the weights in the ball with radius $D=10$ and the Lagrange multipliers to $\|\lambda\|_\infty < \kappa = 1$. Cross-entropy loss is used as a measurement of classification loss, and hinge loss is used as a differentiable convex surrogate of 0-1 loss.

In our experiment, the protected groups are two races,
African-American and Caucasian, and we aim to treat each protected
group equally. For comparison, we consider different methods as follows:
\begin{enumerate}
    \item ``Unconst.'' is the model without any constraint.
    \item ``$T$-Stoch.'' is the model with the fairness constraint and trained by Algorithm~\ref{alg:surrogat-based-optimizer}. It has a set of $T$ parameters and randomly picks one of those parameters uniformly to make a prediction.
    \item  ``$J$-Stoch.''  is the model with the fairness constraints and trained by Algorithm~\ref{alg:surrogat-based-optimizer}. $J$-Stoch also uses random prediction as $T$-Stoch, but only uses $J+1$ parameters at most by the shrinking procedure described in Appendix~\ref{apd:shrinking}.
    \item ``Last'' is the model with the fairness constraint and trained by Algorithm~\ref{alg:surrogat-based-optimizer} and uses the weight in the last iteration.
    \item ``Best'' is the model with fairness constraint and trained by Algorithm~\ref{alg:surrogat-based-optimizer} and uses the weight having the smallest loss.
\end{enumerate}

First, to show the unfair treatment hiding in the classification results, we train a neural network without any constraint. The result is shown in the first line in Table~\ref{table:COMPAS}. We can see that there are no significant differences in accuracy between black and white defendants. However, if we investigate more carefully, there is a huge difference in recall, defined as the positive classification rate of the classifier, which means the classifier fails in different ways for two protected groups and tends to predict black defendants as more likely to reoffend. Note that removing the race feature cannot solve this unfair treatment and would reduce accuracy since there are other features correlated with race.

To address this issue, we define $\Dcal_B$ and $\Dcal_W$ to be conditional distributions given the target of instance is positive and the race of instance is African-American and Caucasian, respectively. Then, our goal is to optimize
\begin{equation} \label{eq:fairBW}
    \begin{split}
        \min_{\thetab \in \Theta} & \  \EE_{(\xb,z)\sim \Dcal_p} h_0(y(\thetab;\xb)), \\
        \text{s.t.} & \  \EE_{ (\xb, z) \sim \Dcal_{B}} \ind \{z = \sgn(y(\thetab; \xb))\} =
        \EE_{ (\xb, z) \sim \Dcal_{W}} \ind \{z = \sgn(y(\thetab; \xb))\},
    \end{split}
\end{equation}
where $\Dcal_p$ is the data distribution. Heuristically, we only record the weights at the end of every epoch and discard the first 1000 iterations. The $T$-stochastic classifier would uniformly pick one of those weights and make a prediction. The solution is further shrank using \eqref{eq:shrinking} to establish a $J$-stochastic classifier, which only uses 2 weights in our experiment.

\begin{table*}[t]
    \centering
    \resizebox{0.95\textwidth}{!}{
    \begin{tabular}{l|rrrrr|rrrrr}
        \toprule
        & \multicolumn{5}{c|}{Train (\%)}            & \multicolumn{5}{c}{Test (\%)}             \\
        Algo.  & A   & A(B) & A(W) & R(B) & R(W) & A   & A(B) & A(W) & R(B) & R(W) \\ \midrule
        Unconst.   & 74.42 & 72.94 & 74.88 & 59.07 & 39.62 & 65.60 & 64.66 & 67.63 & 50.0  & 28.87 \\
        $J$-Stoch. & 75.23 & 71.80 & 77.84 & 55.60 & 55.40 & 65.98 & 64.97 & 64.87 & 48.82 & 41.42 \\
        $T$-Stoch. & 72.31 & 68.14 & 74.14 & 49.86 & 51.97 & 63.44 & 62.59 & 61.34 & 46.27 & 41.00 \\
        Last       & 75.09 & 74.57 & 74.41 & 71.67 & 45.96 & 67.17 & 66.73 & 66.73 & 66.07 & 36.82 \\
        Best       & 75.99 & 73.30 & 77.97 & 61.51 & 53.68 & 66.73 & 66.01 & 66.01 & 55.29 & 40.16 \\ \bottomrule
    \end{tabular}
    }
    \caption{COMPAS Experiment Results. All numbers represent a percentage. A stands for accuracy and R stands for recall. B and W represent black and white defendants, respectively.} \label{table:COMPAS}
\end{table*}

The results are shown in Table~\ref{table:COMPAS} and Figure~\ref{fig:COMPAS}.
We can see that the stochastic classifiers perform similarly to the unconstrained classifier in terms of accuracy, as well as accuracy in different protected groups. Furthermore, the recall of African-Americans aligns with the recall of Caucasian for stochastic classifiers in the training set. For the test set, we also observe a significant improvement in unfair treatment. This means that we fairly treat different subgroups, without losing any predictive power. The same phenomenon cannot be observed for the "Last" and "Best" classifiers. Thus, the randomized model is indeed useful for meeting fairness constraints, in general.

\begin{figure*}[t]
    \centering
    \includegraphics[width=0.47\textwidth]{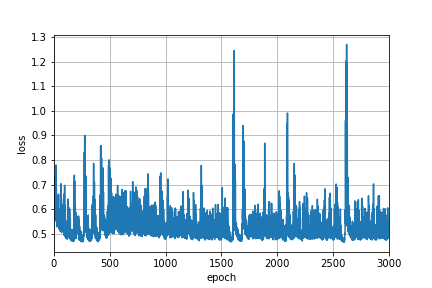}
    \includegraphics[width=0.47\textwidth]{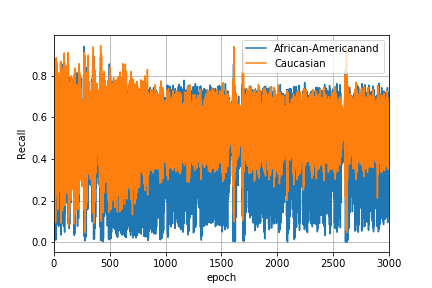}
    \caption{The plots for the loss and recall of two protected groups         for each iteration during training on the COMPAS data set with equal opportunity constraints. We can see that there is some oscillation in two plots. The oscillation caused by conflicting between two constraints in \eqref{eq:fairBW} and violating one of them alternately shows that there is no pure equilibrium to converge to.}
    \label{fig:COMPAS}
\end{figure*}

We also plot the loss and recalls of two protected groups for each of the iterates in Figure~\ref{fig:COMPAS}. We observe oscillation which is caused by violating constraints alternately and suggests that no pure equilibrium exists and that a stochastic classifier may be necessary. The oscillation may commonly occur in practice, especially for optimizing non-convex and non-smooth Lagrangian.

We can understand how randomized models work to achieve the fairness constraint based on this oscillation. Assume that we could find two models with the same high accuracy but different behavior on recalls. One has a high recall on black defendants, and the other has a high recall on white defendants. If we randomly select one model to make a prediction, the average result will balance the recall on two races, which leads to a fair classifier. Neural networks particularly fit this perspective, since overparameterized neural networks can represent many different but high-performing models corresponding to different parameters. We may find several candidate parameters with high accuracy but distinct prediction behaviors when adding various constraints. Then mixing candidate models indeed can improve fairness via the previous argument. The minimax optimization framework can be interpreted as an automated strategy to encourage the neural network to explore different-behavior parameters.

\section{Conclusion}
\label{sec:Discussion}

This work shows how to provably train neural network models under non-convex constraints via a minimax framework. Despite the difficulty of non-convex neural network models, we establish a no-regret bound for online learning of neural networks. Our results shed new light on the theoretical understanding of convergence properties of neural networks with a non-convex constrained optimization problem.

One possible avenue for future work is developing an online procedure \cite{agrawal2014dynamic, li2019online} for shrinkage with theoretical guarantees, which can also reduce memory costs during training. Moreover, randomized prediction violates different principles of fairness, such as ``similar individuals receive similar outcomes'' \cite{dwork2012fairness}. 
\cite{cotter2019making} discussed clustering and ensemble procedures to decrease randomness while satisfying the fairness defined from an individual perspective, while it may be computationally intractable for large models. It may be interesting to find more memory and computationally efficient algorithms for neural networks.



\appendix

\section{Shrinking}
\label{apd:shrinking}

From Proposition~\ref{thm:equilibrium-optimization} we observe that a stochastic model is required to achieve a near-optimal and near-feasible solution. In our setting, the prediction model is parameterized by a neural networks, which may have a considerable amount of parameters and need a large number of iterations for training. As a result, storing all parameters for all iterations for randomized prediction may be intractable. To overcome this difficulty, it can be shown that a smaller mixed equilibrium exists and can be found by using a shrinkage procedure. In particular, let $\thetab^1, \dots, \thetab^t \in \Theta$ be a sequence of $T$ iterates obtained by Algorithm~\ref{alg:surrogat-based-optimizer}. Define $\cbb_0=(h_0(\thetab_1),\dots, h_0(\thetab^T))^\top$ and  $\cbb_j=(g_j(\rb(\thetab^1)),\dots, g_j(\rb(\thetab^T)))^\top$ for $j=1,\dots,J$. The shrinking procedure aims to solve the following linear programming problem:
\begin{equation} \label{eq:shrinking}
	\begin{split}
		\min_{\pb \in T\text{-dimensional simplex}} & \cbb_0^\top \pb \\
		\text{s.t. } & \cbb_j^\top \pb \leq \epsilon, \quad \text{for } j=1,\dots, J,
	\end{split}
\end{equation}
for some $\epsilon>0$. The minimizer $\pb^\ast=(p_1^\ast, \dots, p_T^\ast)$ of \eqref{eq:shrinking} represents the final stochastic classifier, with components representing the probability of sampling $\{1,\dots, T\}$.
Note that due to the convexity of $g_j$, we have for all $j$ that
\[
g_j \left( \sum_{t=1}^T p_t^\ast \rb(\thetab^t) \right) \leq \cbb_j^\top \pb^\ast \leq \epsilon,
\] 
implying that the stochastic solution induced by $\pb^\ast$ is near-feasible in~\eqref{eq:optimization-2} and achieves the smallest training error among all stochastic models.

\cite{cotter2019optimization} showed that every vertex of the feasible region and the optimal solution $\pb^\ast$ have at most $J+1$ nonzero elements. That is, the shrinkage procedure selects at most $J+1$ iterations for constructing randomized solution and there are only $J+1$ non-zero elements in $\pb^\ast$, which reduces memory cost significantly. Since $J$ is much smaller than $T$, heuristically, we only sample a small amount of iterates compared to $T$ in our experiment where $J$ is the number of constraints and $T$ is the total number of iterates.

\section{Proofs}

\subsection{Proof of Proposition~\ref{prop:y-bd}} \label{apd:y-bd}

We first show that $\EE |y(\thetab^0;\xb)|$ is finite and then apply Markov's inequality. Since $\|\xb\|\leq 1$, we have
\begin{equation} \label{eq:moment}
	\EE (\sigma((\ab_i^0)^\top \xb))^2 \leq \EE ((\ab_i^0)^\top \xb)^2 \leq \EE \|\ab_i^0\|^2 = 1.
\end{equation}
Thus, using Khintchine's inequality (Theorem 1 in Chapter 10.3 in \cite{chow2003probability}) and the fact that $b_i \sim \operatorname{Uniform}(\{-1,1\})$, there exists a constant $B$ such that
\begin{align}
 \EE |y(\thetab^0;\xb)| & = \EE \left|\frac{1}{\sqrt{m}} \sum_{i=1}^{m} b_i \sigma((\ab_i^0)^\top \xb) \right| \\
& = \EE \left[ \EE \left|y=\frac{1}{\sqrt{m}} \sum_{i=1}^{m} b_i \sigma((\ab_i^0)^\top \xb) \right| \mid \ab_i^0, \xb \right] \\
& \leq \frac{B}{\sqrt{m}} \EE \left(\sum_{i=1}^m (\sigma((\ab_i^0)^\top \xb))^2\right)^{1/2} && (\text{Khintchine inequality})\\
& \leq \frac{B}{\sqrt{m}} \left( \EE \sum_{i=1}^m (\sigma((\ab_i^0)^\top \xb))^2\right)^{1/2}  && (\text{Jensen's inequality})\\
& \stackrel{\eqref{eq:moment}}{\leq} B.
\end{align}
Markov's inequality implies that $\PP(|y|>D)\leq {B}/{D}$.
Moreover, we obtain
\begin{align}
| &y(\thetab;\xb) - y(\thetab^0;\xb)|^2 \\
& = \left| \frac{1}{\sqrt{m}} \sum_{i=1}^{m} b_i \sigma(\ab_i^\top \xb)- \frac{1}{\sqrt{m}} \sum_{i=1}^{m} b_i \sigma((\ab_i^0)^\top \xb) \right|^2 \\
&\leq \sum_{i=1}^{m} \left| \sigma(\ab_i^\top \xb) - \sigma((\ab_i^0)^\top \xb) \right|^2 && (\text{since } (\sum_{i=1}^{m} c_i)^2 \leq m \sum_{i=1}^m c_i^2)\\
&\leq \sum_{i=1}^{m} \left| \ab_i^\top \xb - (\ab_i^0)^\top \xb \right|^2  && (\text{$\sigma$ is 1-Lipschitz continuous})\\
& \leq \sum_{i=1}^{m} \| \ab_i - \ab_i^0\|^2 && (\|\xb\|\leq 1) \\
& \leq D^2.
\end{align}
The proposition follows by the triangle inequality.

\subsection{Proof of Proposition~\ref{thm:equilibrium-optimization}} \label{apd:equilibrium-optimization}

We prove two guarantees in Proposition~\ref{thm:equilibrium-optimization}: optimality in \eqref{eq:optimal} and feasibility in \eqref{eq:feasible} by properties \eqref{eq:OL} of the coarse-correlated equilibrium. First, we define $\bar{\lambdab} =\frac{1}{T} \sum_{t=1}^T \lambdab^t$ and $\bar{\xib} =\frac{1}{T} \sum_{t=1}^T \xib^t$.
	
{\bf Optimality.} By Proposition~\ref{prop:y-bd}, we have $|y(\thetab;\xb)| <2D$ with probability $1-B/D$. Therefore, on the event $\{ |y(\thetab;\xb)| <2D \}$ we can always find a $\xib \in \Xi$ such that $\tilde{\rb}(\thetab)\leq \xib$. Then, we know that the Lagrangian gives a lower bound for the original problem. That is, for any $\lambdab \geq \zero$, we have 
\begin{equation} \label{eq:proposition3}
\begin{split}
	\min_{\thetab \in \Theta, \xib \in \Xi} \tilde{\Lcal}(\thetab,\xib,\lambdab) 
	&\leq \min_{\thetab \in \Theta, \xib \in \Xi: \gb(\xib)\leq \zero, \tilde{\rb}(\thetab)\leq \xib } \tilde{\Lcal}(\thetab,\xib,\lambdab) \\
	&\leq \min_{\thetab \in \Theta, \xi \in \Xi: \gb(\xib)\leq \zero, \tilde{\rb}(\thetab)\leq \xib} h_0(\thetab) \\
	&= \min_{\thetab \in \Theta: \max_j g_j (\tilde{\rb}(\thetab))} h_0(\thetab),
\end{split}
\end{equation}
where the first inequality follows due to the fact that we restrict the search space, the second inequality holds because of $\lambdab \geq \zero, \gb(\xib)\leq \zero, \tilde{\rb}(\thetab)\leq \xib$, and the equality follows from the monotonicity of $g_j$.
	
From the definition of the equilibrium, we obtain
\begin{align} 
\frac{1}{T} 
\sum_{t=1}^T \Lcal (\thetab^t, \xib^t, \lambdab^t) 
& \stackrel{\eqref{eq:upper-bound}}{\leq} \frac{1}{T} \sum_{t=1}^T \tilde{\Lcal} (\thetab^t, \xib^t, \lambdab^t). \notag \\
\intertext{Using \eqref{eq:OL-xi} and \eqref{eq:OL-theta}, we further have}
& \stackrel{}{\leq} \min_{\xib \in \Xi} \frac{1}{T} \sum_{t=1}^T \Lcal_1(\xib,\lambdab^t) + \epsilon_\xi 
+ \min_{\thetab \in \Theta} \frac{1}{T} \sum_{t=1}^T \tilde{\Lcal}_2(\thetab,\lambdab^t)  + \epsilon_\theta \notag \\  
& = \min_{\thetab \in \Theta, \xib \in \Xi} \tilde{\Lcal} (\thetab, \xib, \bar{\lambdab}) + \epsilon_\theta + \epsilon_\xi \\
& \leq \min_{\thetab \in \Theta: \max_j g_j(\tilde{\rb}(\thetab))\leq 0} h_0(\thetab) + \epsilon_\theta + \epsilon_\xi, 
 \label{eq:theta-xi-equilibruim}
\end{align}
where the equality holds due to the linearity of $\tilde{\Lcal}$ with respect to $\lambdab$,
and the last line follows from $\bar{\lambdab}>\zero$ and \eqref{eq:proposition3}.
	
We have for any $\lambdab^\prime \in \Lambda$
\begin{equation} \label{eq:lagrangian-upper-bound}
\begin{split}
	\frac{1}{T} \sum_{t=1}^T \Lcal (\thetab^t, \xib^t, \lambdab^\prime) 
	&\leq  \max_{\lambdab \in \Lambda} \frac{1}{T} \sum_{t=1}^T \Lcal (\thetab^t, \xib^t, \lambdab) \\ 
	& \stackrel{\eqref{eq:OL-lambda}}{\leq} \frac{1}{T} \sum_{t=1}^T {\Lcal}(\thetab^t, \xib^t, \lambdab^t) + \epsilon_\lambda \\
	&\stackrel{\eqref{eq:theta-xi-equilibruim}}{\leq}  \min_{\thetab \in \Theta: \max_j g_j(\tilde{\rb}(\thetab))\leq 0} h_0(\thetab) + \epsilon_\theta + \epsilon_\xi + \epsilon_\lambda.
\end{split}
\end{equation}
The optimality follows by setting $\lambdab^\prime = \zero$.
	
{\bf Feasibility.} Letting $j^\prime \in \argmax_j g_j(\bar{\xib})$ and setting $\lambda_{j^\prime}^\prime =\kappa$ and $\lambda^\prime_j=0$ for $j\neq j^\prime$, we get
\begin{equation} \label{eq:optimality}
\begin{split}
	 \frac{1}{T} \sum_{t=1}^T h_0(\thetab^t) + \frac{\kappa}{T} \sum_{t=1}^T g_{j^\prime} (\xib^t)
	&\stackrel{\eqref{eq:Lagrangian-1}}{=} \frac{1}{T} \sum_{t=1}^T \Lcal (\thetab^t, \xib^t, \lambdab^\prime) \\
	& \stackrel{\eqref{eq:lagrangian-upper-bound}}{\leq} \min_{\thetab \in \Theta: \max_j g_j(\tilde{\rb}(\thetab))\leq 0} h_0(\thetab) + \epsilon_\theta + \epsilon_\xi + \epsilon_\lambda
\end{split}
\end{equation}
and, using convexity of $g_j$,
\begin{align} 
\max_j g_{j}(\bar{\xib}) 
& \leq \frac{1}{T} \sum_{t=1}^T g_{j^\prime} (\xib^t) 
\notag \\ 
& \stackrel{\eqref{eq:optimality}}{\leq} \frac{1}{\kappa} \left( \min_{\thetab \in \Theta: \max_j g_j(\tilde{\rb}(\thetab))\leq 0} h_0(\thetab) - \frac{1}{T} \sum_{t=1}^T h_0(\thetab^t) + \epsilon_\theta + \epsilon_\xi + \epsilon_\lambda \right) 
\notag \\
& \leq \frac{2C + \epsilon_\theta + \epsilon_\xi + \epsilon_\lambda}{\kappa},
\label{eq:ineq1}
\end{align}
where the inequality follows from part 3 of Assumption~\ref{asp:1}.
Similarly, letting $k^\prime \in \argmax_k \left( \frac{1}{T}\sum_t r_k(\thetab^t) - \bar{\xi}_k \right)$ and setting $\lambda_{J+k^\prime}^\prime = \kappa$ and $\lambda_{J+k^\prime} =0$ for $k\neq k^\prime$, we obtain
\begin{multline} \label{eq:optimality-2}
\frac{1}{T} \sum_{t=1}^T h_0(\thetab_t) + \kappa \left( \frac{1}{T} \sum_{t=1}^T r_{k^\prime}(\thetab^t) - \bar{\xi}_{k^\prime} \right) 
 \stackrel{\eqref{eq:Lagrangian-1}}{=} \frac{1}{T} \sum_{t=1}^T \Lcal (\thetab^t, \xib^t, \lambdab^\prime) \\
 \stackrel{\eqref{eq:lagrangian-upper-bound}}{\leq} \min_{\thetab \in \Theta: \max_j g_j(\tilde{\rb}(\thetab))\leq 0} h_0(\thetab) + \epsilon_\theta + \epsilon_\xi + \epsilon_\lambda,
\end{multline}
which implies
\begin{multline}\label{eq:ineq2}
\max_k \frac{1}{T}\sum_{t=1}^T r_{k}(\thetab^t) - \bar{\xi}_k \stackrel{\eqref{eq:optimality-2}}{\leq} ( 2C + \epsilon_\theta + \epsilon_\xi + \epsilon_\lambda ) / \kappa \\
\Longrightarrow \max_k \max\left\{\frac{1}{T}\sum_{t=1}^T r_{k}(\thetab^t) - \bar{\xi}_k, 0 \right\} 
			\leq ( 2C + \epsilon_\theta + \epsilon_\xi + \epsilon_\lambda ) / \kappa.
\end{multline}
Therefore, part 4 of Assumption~\ref{asp:1} implies that
\begin{align*}
\max_j g_j \left(  \frac{1}{T}\sum_{t=1}^T \rb(\thetab^t) \right) 
& \leq  \max_j g_j\left(  \max\left\{ \frac{1}{T}\sum_{t=1}^T \rb(\thetab^t)- \bar{\xib}, 0\right\} + \bar{\xib} \right) \\
\intertext{as $g_j$ is increasing. Furthermore, as $g_j$ is L-Lipschitz and increasing, we have}
& \leq \max_j g_j(\bar{\xib}) + L  \left\|\max\left\{ \frac{1}{T}\sum_{t=1}^T \rb(\thetab^t) - \bar{\xib}, 0 \right\} \right\|_\infty 
\\
& = \max_j g_j(\bar{\xib}) + L  \max_k \max\left\{ \frac{1}{T}\sum_{t=1}^T r_k(\thetab^t) - \bar{\xi}_k, 0 \right\}  \\
	& \leq (1+L) ( 2C + \epsilon_\theta + \epsilon_\xi + \epsilon_\lambda ) / \kappa, 
\end{align*}
where the last line follows from \eqref{eq:ineq1} and \eqref{eq:ineq2}. This completes this proof. 

\subsection{Proof of Proposition~\ref{prop:noise-regret-bound}} \label{apd:noise-regret-bound}

Define $\hat{f}_t(\thetab) = f_t(\thetab^t) + (\mub^t)^\top (\thetab-\thetab^t)$.
Since $\hat{f}_t(\thetab)$ is convex and $\nabla \hat{f}_t(\thetab) = \mub^t $, applying Lemma~\ref{lemma:regret-bound}, we have
\begin{equation} \label{eq:conve-OL}
    \frac{1}{T} \sum_{t=1}^T \hat{f}_t(\thetab^t) - \frac{1}{T} \sum_{t=1}^T \hat{f}_t(\thetab) \leq \frac{\eta W}{2} +  \frac{M}{T\eta},
\end{equation}
where we have used \eqref{eq:boundedness} and convexity of $\Theta$. We know that the convexity of $f_t$ implies (see, for example, \cite{zhou2018fenchel}) the following
\begin{equation} \label{eq:f-convex}
    f_t(\thetab^t) + \nabla f_t(\thetab^t)^\top (\thetab-\thetab^t) \leq f_t(\thetab).
\end{equation} 
Therefore, we have
\begin{align*}
    &\frac{1}{T} \sum_{t=1}^T f_t(\thetab^t) - \frac{1}{T} \sum_{t=1}^T f_t(\thetab) \\
    & \stackrel{\hat{f}_t(\thetab^t) =  f_t(\thetab^t)}{=} \frac{1}{T} \sum_{t=1}^T \hat{f}_t(\thetab^t) - \frac{1}{T} \sum_{t=1}^T \hat{f}_t(\thetab) + \frac{1}{T} \sum_{t=1}^T \hat{f}_t(\thetab) - \frac{1}{T} \sum_{t=1}^T f_t(\thetab) \\
    & \stackrel{\eqref{eq:conve-OL}}{\leq} \frac{\eta W}{2} +  \frac{M}{T\eta} + \frac{1}{T} \sum_{t=1}^T \hat{f}_t(\thetab) - \frac{1}{T} \sum_{t=1}^T f_t(\thetab) \\
    & \stackrel{\eqref{eq:f-convex}}{\leq} \frac{\eta W}{2} +  \frac{M}{T\eta} + \frac{1}{T} \sum_{t=1}^T (\mub^t - \nabla f_t(\thetab^t))^\top(\thetab-\thetab^t) \\
    & = \frac{\eta W}{2} +  \frac{M}{T\eta} + \frac{1}{T} \sum_{t=1}^T (\mub^t - \nabla f_t(\thetab^t)-\betab_t(\thetab^t))^\top(\thetab-\thetab^t) \\
    & \qquad \qquad + \frac{1}{T} \sum_{t=1}^T \betab_t(\thetab^t)^\top (\thetab-\thetab^t).
\end{align*}
Since the function $h$ is differentiable and strongly convex, there exists $\thetab^\prime$ such that $\nabla h(\thetab^\prime) =\zero$. Thus, for all $t$, we have
\begin{multline} \label{eq:theta-bound}
\frac{1}{4}\| \thetab - \thetab^t \|^2 
 \leq \frac{1}{2}\| \thetab - \thetab^\prime \|^2 + \frac{1}{2} \| \thetab^t - \thetab^\prime \|^2 \leq B_h(\thetab, \thetab^\prime) + B_h(\thetab^t, \thetab^\prime) \\ 
\stackrel{\nabla h(\thetab^\prime) =\zero}{\leq} 2 \sup_{\thetab \in \Theta}  h(\thetab) \stackrel{\eqref{eq:boundedness}}{\leq} 2M,
\end{multline}
where the first inequality holds by $(a+b)^2\leq 2 a^2 + 2 b^2$ 
and the second inequality follows by the fact that $h$ is 1-strongly convex w.r.t $\|\cdot\|$.

Using Jensen's inequality, we have
\begin{equation} \label{eq:jensen}
    \sup_t \|\nabla f_t(\thetab^t)+\betab_t(\thetab^t)\|_\ast = \sup_t \|\EE[\mub^t|\thetab^t]\|_\ast \leq \EE[\sup_t \|\mub^t\|_\ast|\thetab^t] \stackrel{\eqref{eq:boundedness}}{\leq}  W,
\end{equation} 
and Holder's inequality implies that 
\begin{equation} \label{eq:holder}
    \begin{split}
        & \ \ \ |(\mub^t - \nabla f_t(\thetab^t) - \betab_t(\thetab^t))^\top(\thetab-\thetab^t)| \\
        &\stackrel{\eqref{eq:theta-bound}}{\leq} \| \mub^t - \nabla f_t(\thetab^t)  -\betab_t(\thetab^t) \|_\ast \cdot 2 \sqrt{2M} \\
        &\leq 
        (\| \mub^t\|_\ast + \| \nabla f_t(\thetab^t) + \betab_t(\thetab^t) \|_\ast) \cdot 2 \sqrt{2M} \stackrel{\eqref{eq:jensen}, \eqref{eq:theta-bound}}{\leq} 4W\sqrt{2M}.
    \end{split}
\end{equation}	
Furthermore, due to \eqref{eq:holder} and the fact that 
\[\EE [(\mub^t - \nabla f_t(\thetab^t)-\betab_t(\thetab^t))^\top(\thetab-\thetab^t)|\thetab^t] =0,
\]
$(\mub^t - \nabla f_t(\thetab^t)-\betab_t(\thetab^t))^\top(\thetab-\thetab^t)$ is a bounded martingale difference. Therefore, by the Hoeffding-Azuma inequality, we get
\begin{equation*}
    \PP \left\{\sum_{t=1}^T (\mub^t - \nabla f_t(\thetab^t) - \betab_t(\thetab^t))^\top(\thetab-\thetab^t) \geq \varepsilon \right\} \leq \exp\left( \frac{-\varepsilon^2}{64 T W^2M} \right).
\end{equation*}
The proposition follows by setting $\varepsilon = 8W\sqrt{T M \log (1/\delta)}$ and  $\eta = \sqrt{\frac{M}{TW}}$.

\subsection{Proof of Proposition~\ref{prop:linearization}} \label{apd:linearization}

Since $\|\xb\|_2\leq 1$ and $b_i \in \{-1, 1\}$, 
we have $\| \nabla y(\thetab;\xb) \|_2 \leq 1$ as
\begin{equation} \label{eq:grad-y}
    \nabla y(\thetab;\xb) = \frac{1}{\sqrt{m}} \left( \ind\{ \ab_1^\top \xb >0 \} \xb^\top , \dots, \ind\{ \ab_m^\top \xb >0 \} \xb^\top \right)^\top.
\end{equation}
Following \cite{liu2019neural}, we prove the local linearization for neural network models. Given $\xb \in \Xcal$ and a layer $i$ such that $\ind\{\ab_i^\top \xb>0\} \neq \ind\{(\ab_i^0)^\top \xb>0\}$, applying the Cauchy–Schwarz inequality gives us
\begin{subequations}\label{eq:ind}
\begin{align} 
    & \quad \quad (\ab_i^\top \xb)((\ab_i^0)^\top \xb) <0 \label{eq:cond1} \\ 
    & \Longrightarrow |(\ab_i^0)^\top \xb| \stackrel{\eqref{eq:cond1}}{\leq} |(\ab_i^0 - \ab_i)^\top \xb| \stackrel{\|\xb\|_2 \leq 1}{\leq} \| \ab_i^0 - \ab_i \|_2.
\end{align}
\end{subequations}
Thus, we get
\begin{equation} \label{eq:step1}
    \begin{split}
        & \ \ \ |y(\thetab;\xb)-y^0(\thetab;\xb)| \\
        &= \frac{1}{\sqrt{m}} \left|\sum_{i=1}^m b_i (\ind\{\ab_i^\top \xb>0\} - \ind\{(\ab_i^0)^\top \xb>0\}) \ab_i^\top \xb \right| \\
        & \stackrel{\|\xb\|_2 \leq 1}{\leq} \frac{1}{\sqrt{m}} \sum_{i=1}^m \left| \ind\{\ab_i^\top \xb>0\} - \ind\{(\ab_i^0)^\top \xb>0\}\right| (|(\ab_i^0)^\top \xb| + \|\ab_i-\ab_i^0\|_2)  \\
        & \stackrel{\eqref{eq:ind}}{\leq} \frac{1}{\sqrt{m}} \sum_{i=1}^m \ind\{ |(\ab_i^0)^\top \xb| \leq \|\ab_i^0-\ab_i\|_2\}  (|(\ab_i^0)^\top x| + \|\ab_i-\ab_i^0\|_2)  \\
        &\leq \frac{2}{\sqrt{m}}  \sum_{i=1}^m \ind\{ |(\ab_i^0)^\top \xb| \leq \| \ab_i^0 -\ab_i \|_2 \}  \| \ab_i - \ab_i^0 \|_2,
    \end{split}
\end{equation}
where we have used $\ind\{|c|\leq d\}|c| \leq \ind\{|c|\leq d\} d$ for the last inequality.
	
For $v \sim \chi_d^2$, where $\chi^2_d$ denotes a chi-squared distribution with $d$ degrees of freedom, we have $\EE[ 1/v] = 1/(\text{deg}-2)$. Therefore, we have
\begin{equation} \label{eq:inverse-chi-squared distribution}
    \EE \|\ab_i^0\|_2^{-2} = \frac{1}{d-2}.
\end{equation}
Furthermore, applying Cauchy–Schwarz inequality \eqref{eq:step1}, we get
\begin{align}
\EE_{\xb, \thetab^0} &|y(\thetab;\xb)-y^0(\thetab;\xb)|^2 \\
& \leq \frac{4}{m} \EE_{\xb, \thetab^0} \left( \sum_{i=1}^m \ind\{ |(\ab_i^0)^\top x| \leq \| \ab_i^0 -\ab_i \|_2 \} \right) \left( \sum_{i=1}^m \| \ab_i - \ab_i^0 \|_2^2 \right).
\\
\intertext{Since $\|\thetab-\thetab^0\|_2 \leq D$, we further have}
& \leq \frac{4 D^2}{m} \sum_{i=1}^m \EE_{\xb, \thetab^0} \ind\{ |(\ab_i^0)^\top \xb| \leq \| \ab_i^0 -\ab_i \|_2 \}  .
\\
\intertext{Using Part 2 in Assumption~\ref{assumption:Regularity-Data} and the Cauchy–Schwarz inequality again, we have}
& \leq \frac{4 c D^2}{m}  \EE_{\thetab^0} \left[ \sum_{i=1}^m \frac{\| \ab_i^0 -\ab_i \|_2}{\|\ab_i^0\|_2} \right]  \\
& \leq \frac{4 c D^2}{m} \EE_{\thetab^0} \left( \sum_{i=1}^m \|\ab_i^0\|_2^{-2} \right)^{1/2} \left(\sum_{i=1}^m \| \ab_i^0 -\ab_i \|_2^2 \right)^{1/2} .
\\
\intertext{Using $\|\thetab-\thetab^0\|_2 \leq D$ we have}
& \leq \frac{4 c D^3}{m} \left( \sum_{i=1}^m \EE_{\thetab^0} \|\ab_i^0\|_2^{-2} \right)^{1/2} \\
& = \frac{4 c D^3}{m^{1/2}(d-2)^{1/2}},
\end{align}
where we have used Jensen's inequality for the square root function, expectation in the second to last inequality,
and \eqref{eq:inverse-chi-squared distribution} at the end.
Following a similar argument, we have
\begin{align*}
    \EE_{\xb, \thetab^0} &\| \nabla y(\thetab;\xb) - \nabla y^0(\thetab;\xb) \|_2 \\
    &\stackrel{\eqref{eq:grad-y}}{=} \frac{1}{m} \EE_{\xb, \thetab^0} \left\| \sum_{i=1}^m b_i (\ind\{\ab_i^\top \xb>0\} - \ind\{(\ab_i^0)^\top \xb>0\})  \xb \right\|_2 
     \\
    &\stackrel{\eqref{eq:ind}}{\leq} \frac{1}{m} \EE_{\xb, \thetab^0} \sum_{i=1}^m \ind\{ |(\ab_i^0)^\top \xb| \leq \| \ab_i^0 -\ab_i \|_2 \} \|\xb\|_2  
    \\
    & \leq \frac{1}{m} \EE_{\xb, \thetab^0} \sum_{i=1}^m \ind\{ |(\ab_i^0)^\top \xb| \leq \| \ab_i^0 -\ab_i \|_2 \}  .
    \intertext{Using $\|\xb\|_2\leq 1$ and part 2 in Assumption~\ref{assumption:Regularity-Data}, we have}
    \\
    & \leq \frac{c}{m} \EE_{\thetab^0} \sum_{i=1}^m \frac{\| \ab_i^0 -\ab_i \|_2}{\|\ab_i^0\|_2}      
    \\
    & \leq \frac{D c}{m} \left( \sum_{i=1}^m \EE_{\thetab^0} \|\ab_i^0\|_2^{-2} \right)^{1/2} \\
    & \stackrel{\eqref{eq:inverse-chi-squared distribution}}{=} \frac{Dc}{m^{1/2}(d-2)^{1/2}},    
\end{align*}
where the last inequality follows from Cauchy–Schwarz inequality, $\|\thetab-\thetab^0\|_2 \leq D$, and Jensen's inequality. This completes the proof.

\subsection{Proof of Theorem~\ref{thm:OL-NN}} \label{apd:OL-NN}

From the decomposition~\eqref{eq:decomposition}, it suffices to bound three terms $(I), (III), (III)$. The $(I)$ and $(III)$ in \eqref{eq:decomposition} can be bounded in a uniform manner. Note that $y^0$ depends on $\thetab^0$. 
In particular, given any $\thetab \in \Theta$,
since $f_t$ is convex for all $t$, we have that
\begin{align} 
\EE_{\xb, z, \thetab^0} &\left[ f_t(y(\thetab;\xb),z) - f_t(y^0(\thetab;\xb),z) \right] \\
&\stackrel{}{\leq} \sup_{t, y, z} |\nabla_y f_t(y, z)| \times \EE_{\xb,\thetab^0} | (y(\thetab;\xb) - y^0(\thetab;\xb))| 
\\
&\stackrel{\eqref{eq:thm-asp}}{\leq} L_f \EE_{x,\thetab^0} |y(\thetab^t;\xb) - y^0(\thetab^t;\xb)|. \\
\intertext{Using Jensen’s Inequality and Part 2 in Proposition~\ref{prop:linearization}, we have}
&\stackrel{}{\leq} L_f \left(\EE_{x,\thetab^0} | (y(\thetab^t;\xb) - y^0(\thetab^t;\xb))|^2 \right)^{1/2} 
\\
&\stackrel{}{=} O \left(\frac{L_f D^{3/2}}{m^{1/4}}\right).
\label{eq:I-III}
\end{align}
This implies that
\begin{equation}
\begin{split}
	(III) & = \frac{1}{T} \sum_{t=1}^T \EE_{\xb,z, \thetab^0}[f_t(y^0(\thetab;\xb),z)] - \frac{1}{T} \sum_{t=1}^T \EE_{\xb,z, \thetab^0}[f_t(y(\thetab;\xb),z)] \\
	& \leq \left| \frac{1}{T} \sum_{t=1}^T \EE_{\xb,z, \thetab^0}[f_t(y^0(\thetab;\xb),z)] - \frac{1}{T} \sum_{t=1}^T \EE_{\xb,z, \thetab^0}[f_t(y(\thetab;\xb),z)] \right| \\
	& \stackrel{}{\leq} \sup_t \EE_{\xb,z, \thetab^0}| [f_t(y^0(\thetab;\xb),z)] - f_t(y(\thetab;\xb),z)|.
\end{split}
\end{equation}
The term $(I)$ is controlled in a similar manner, and we get
\begin{equation*}
	(I) =  O \left(\frac{L_f D^{3/2}}{m^{1/4}}\right), \ \ \ (III) =  O \left(\frac{L_f D^{3/2}}{m^{1/4}}\right).
\end{equation*}

To apply Proposition~\ref{prop:noise-regret-bound} for the second term $(II)$, let $h(\thetab) = \frac{1}{2}\|\thetab -\thetab^0\|^2$. Then, $h^\ast(\thetab) = \frac{1}{2}\|\thetab\|^2 + \thetab^\top\thetab^0$, $B_h(\thetab, \thetab^\prime) = \frac{1}{2}\|\thetab - \thetab^\prime\|^2$, and \eqref{eq:md} reduces to Algorithm~\ref{alg:surrogat-based-optimizer}. We have $\sup_{\thetab \in \Theta} h(\thetab) = M \leq D^2/2$. Then, it suffices to show $\sup_t \|\mub^t \|_2 = W <L_f$,
where $\mub^t = \nabla f_t(y(\thetab^t;\xb_{t+1}), z_{t+1})$, 
\begin{equation}
    \EE \left[ \mub^t \mid \thetab^t \right] = \EE_{x,z,\thetab^0} [ \nabla f_t(y^0(\thetab^t;\xb), z) \mid \thetab^t] + \betab_t(\thetab^t),
\end{equation}
and
\begin{equation}
    \betab_t(\thetab^t) = \EE_{x,z,\thetab^0} [\nabla f_t(y(\thetab^t;\xb), z)  -  \nabla f_t(y^0(\thetab^t;\xb), z) \mid \thetab^t].
\end{equation}
We have
\begin{align*}
    \sup_t \|\mub^t \|_2
    & = \sup_t \| \nabla f_t(y(\thetab^t;\xb_{t+1}),z_{t+1}) \|_2 \\
    & = \sup_t \| \nabla_y f_t(y(\thetab^t;\xb_{t+1}),z_{t+1}) \nabla_{\thetab} y(\thetab^t;\xb_{t+1})\|_2, \\
\intertext{which can further be bounded using Part 1 in Proposition~\ref{prop:linearization} as}    
    & \stackrel{}{\leq}  \sup_{t, y \in \Ycal ,z \in \Zcal} | \nabla_y f_t(y,z) | 
    \\
    & \stackrel{\eqref{eq:thm-asp}}{\leq}  L_f.
\end{align*}
Applying Proposition~\ref{prop:noise-regret-bound} with $\eta = \sqrt{\frac{M}{TW}} = \sqrt{\frac{D^2}{2 L_f T}}$, we get
\begin{align*}
    (II) 
    &= \frac{1}{T} \sum_{t=1}^T \EE_{x, z, \thetab^0} [f_t(y^0(\thetab^t;\xb),z)] -  \frac{1}{T}  \sum_{t=1}^T  \EE_{x, z, \thetab^0} [f_t(y^0(\thetab;\xb),z)] \\
    &\leq \frac{3}{2} \sqrt{\frac{M W}{T}} + 8 W \sqrt{\frac{M \ln (1/\delta)}{T}}  + \frac{2\sqrt{2M}}{T} \sum_{t=1}^T \| \betab_t(\thetab^t) \|_2 \\
    &=  \frac{3D}{2}\sqrt{\frac{L_f}{2 T}} + 8L_f D \sqrt{\frac{\ln(1/\delta)}{2T}} + \frac{2D}{T} \sum_{t=1}^T \|\beta_{t}(\thetab^t)\|_2,
\end{align*}
with probability $1-\delta$.
Furthermore, using Jensen's inequality, we have
\begin{align*}
    \|\betab_t(\thetab)\|_2
    & \stackrel{\text{ }}{\leq} \EE_{x,z,\thetab^0} \| \nabla f_t(y(\thetab;\xb),z)- \nabla f_t(y^0(\thetab;\xb),z) \|_2 \\
    &= \EE_{x,z,\thetab^0} \| \nabla_y f_t(y(\thetab;\xb),z) \nabla_{\thetab} y(\thetab;\xb) 
    - \nabla_y f_t(y^0(\thetab;\xb),z) \nabla_{\thetab} y^0(\thetab;\xb)\|_2 \\
    & \leq \EE_{x,z,\thetab^0} | \nabla_y f_t(y(\thetab;\xb),z)-\nabla_y f_t(y^0(\thetab;\xb),z) | \| \nabla_{\thetab} y(\thetab;\xb) \|_2 \\
    &\quad  + \EE_{x,z,\thetab^0}|\nabla f_t(y^0(\thetab;\xb),z) | \| \nabla_{\thetab} y(\thetab;\xb)- \nabla_{\thetab} y^0(\thetab;\xb) \|_2, \\
\intertext{which can further be bounded by Part 1 in Proposition~\ref{prop:linearization} as }    
    & \leq \EE_{x,z,\thetab^0} | \nabla_y f_t(y(\thetab;\xb),z)-\nabla_y f_t(y^0(\thetab;\xb),z) |  \\
    & \qquad \qquad  + \EE_{x,z,\thetab^0}|\nabla f_t(y^0(\thetab;\xb),z) | \| \nabla_{\thetab} y(\thetab;\xb)- \nabla_{\thetab} y^0(\thetab;\xb) \|_2 \\
    &\stackrel{\eqref{eq:thm-asp}}{\leq} 2L_f^2 (\EE_{x,z,\thetab^0}| y(\thetab^t;\xb)- y^0(\thetab^t;\xb) | \\
    & \qquad \qquad + \EE_{x,z,\thetab^0}\| \nabla_{\thetab} y(\thetab^t;\xb)- \nabla_{\thetab} y^0(\thetab^t;\xb) \|_2)  \\
    &\stackrel{}{=} O \left( L_f^2 \left[ \frac{ D^{3/2}}{m^{1/4}} + \frac{ D}{m^{1/2}} \right] \right) = O \left(   \frac{L_f^2 D^{3/2}}{m^{1/4}} \right),
\end{align*}
where we used \eqref{eq:I-III} and Parts 2 and 3 in Proposition~\ref{prop:linearization} at the end.
Combining all the results, we get
\begin{equation*}
		\begin{split}
			\frac{1}{T} \sum_{t=1}^T &\EE_{x, z, \thetab^0} [f_t(y^0(\thetab^t;\xb),z)] -  \min_\theta\frac{1}{T}  \sum_{t=1}^T  \EE_{x, z, \thetab^0} [f_t(y^0(\thetab;\xb),z)] \\
			&= O \left( D\sqrt{\frac{L_f}{T}} + L_f D \sqrt{\frac{\ln(1/\delta)}{T}} + \frac{L_f^2 D^{3/2}}{m^{1/4}} \right), 
    \end{split}
\end{equation*}
which finishes the proof.

\subsection{Proof of Theorem~\ref{thm:main-thm}} \label{apd:main-thm}
	
Since \eqref{eq:optimal-2} and \eqref{eq:feasible-2} immediately follow by Proposition~\ref{thm:equilibrium-optimization}, it suffices to prove the regret bounds in \eqref{eq:equilibrium}.
For $\thetab$, we will use the result in Theorem~\ref{thm:OL-NN} and justify its conditions.
For $\xib$ and $\lambdab$, we could apply Proposition~\ref{prop:noise-regret-bound}, because their objective functions are convex.
	
We first prove that $\thetab^t$ achieves its approximate coarse-correlated equilibrium in \eqref{eq:OL-theta}. It suffices to verify condition \eqref{eq:thm-asp} in Theorem~\ref{thm:OL-NN}, and that the boundedness condition is satisfied by the assumption on $\Theta$ in the first statement in Assumption~\ref{asp:1} by letting $h(\thetab) = \frac{1}{2} \|\thetab-\thetab^0\|^2$. Let
\begin{equation}
    f_t(y, z) = h_0(y,z) + \sum_{k=1}^K \lambda_{J+k}^t \tilde{h}_k(y, z).
\end{equation}
By Assumption~\ref{assumption:Regularity-OC}, we have
\begin{equation*}
    \begin{split}
        |\nabla_y f_t(y,x)|
        \leq \left| \nabla_y h_0(y,z) \right| +  \sum_{k=1}^K \lambda_{J+k}^t \left| \nabla_y \tilde{h}_k(y,z) \right| 
        = (1+K\kappa) L \leq 2K\kappa L
    \end{split}
\end{equation*}
and
\begin{equation*}
    \begin{split}
        |\nabla_y &f_t(y_1,x) - \nabla_y f_t(y_2,x)| \\
        &\leq \left| \nabla_y h_0(y_1,z) - \nabla_y h_0(y_2,z)\right| +  \sum_{k=1}^K \lambda_{J+k}^t \left| \nabla_y \tilde{h}_k(y_1,z) - \nabla_y \tilde{h}_k(y_2,z) \right| \\
        &= (1+K\kappa) L |y_1-y_2| \leq  2 K L \kappa |y_1-y_2|.
    \end{split}
\end{equation*}
Therefore, using Theorem~\ref{thm:OL-NN}, with probability $1-\delta$ and $\eta_{\theta}= \sqrt{D^2/[4K\kappa L T] } $, we have
\begin{multline*}
    \frac{1}{T} \sum_{t=1}^T  \tilde{\Lcal}_2(\thetab^t, \lambdab^t) - \min_{\thetab \in \Theta} \frac{1}{T} \sum_{t=1}^T \tilde{\Lcal}_2(\thetab,\lambdab^t) 
    \\
    = O\left( K L \kappa D \sqrt{\frac{\ln(1/\delta)}{T}} +  \frac{ (K L \kappa)^2 D^{3/2}}{m^{1/4}}\right),
\end{multline*}
which proves \eqref{eq:OL-theta} with 
\begin{equation}
    \epsilon_\theta 
    = O\left( \kappa D \sqrt{\frac{\ln(1/\delta)}{T}} +  \frac{  \kappa^2 D^{3/2}}{m^{1/4}}\right).
\end{equation}
	
Then, we prove \eqref{eq:OL-xi} for $\xib$. 
Note that Proposition~\ref{prop:y-bd} implies that 
\[
\sup_{\thetab \in \Theta}|y(\thetab;\xb)|<2D
\]
with probability $1-B/D$. On the event $\sup_{\thetab \in \Theta} y(\thetab;\xb) \in \Ycal = \{y:|y|<2D\}$, Assumption~\ref{assumption:Regularity-OC} implies
	\begin{equation*}
	\begin{split}
		& \ \ \sup_t \|\nabla_{\xib} \Lcal_1(\xib^t,\lambdab^t)\|_2 \\
		& \leq \sup_{t, \xib \in \Xi, \lambdab \in \Lambda} \left\| \left( \sum_{j=1}^J \lambda_{j} \nabla_{\xi_1} g_j(\xib) - \lambda_{J+1}, \dots, \sum_{j=1}^J \lambda_{j} \nabla_{\xi_K} g_j(\xib) - \lambda_{J+K} \right)^\top \right\|_2 \\ 
		&\leq \sup_{t, \xib \in \Xi, \lambdab \in \Lambda} \sum_{k=1}^K \left| \sum_{j=1}^J \lambda_{j} \nabla_{\xi_k} g_j(\xib) -  \lambda_{J+k} \right| \\
		&\leq \kappa K(J L+1) \leq 2 \kappa J L K.
	\end{split}
	\end{equation*} 
By Assumption~\ref{asp:1}, we know $\sup_{y \in \Ycal, z \in \Zcal} |\tilde{h}_k| \leq C$. Thus, Proposition~\ref{prop:noise-regret-bound} with $M = C\sqrt{K}, W= 2 \kappa J L K$ implies
\begin{equation}
    \frac{1}{T} \sum_{t=1}^T  \Lcal_1(\xib^t, \lambdab^t) - \min_{\xib \in \Xi} \frac{1}{T} \sum_{t=1}^T \Lcal_1(\xib, \lambdab^t) 
    = O\left(  \kappa J L K^{5/4} \sqrt{\frac{C \ln(1/\delta)}{T}}\right),
\end{equation}
with probability $1-\delta$ and $\eta_\xi = \sqrt{\frac{C}{2 T\kappa J L K^{1/2}}}$. 
Therefore, \eqref{eq:OL-xi} follows with
\begin{equation}
    \epsilon_\xi = O\left(\frac{\kappa   \sqrt{\ln(1/\delta)}}{\sqrt{T}} \right).
\end{equation}
	
Finally, we prove that $\lambdab^t$ satisfies \eqref{eq:OL-lambda}.
We have
\begin{equation*}
\begin{split}
    \sup_t &\|\hat{\nabla}_{\lambdab} \Lcal(\thetab^t, \xib^t,\lambdab^t)\|_2 \\ 
    &= \sup_{\thetab \in \Theta, \xib \in \Xi} \| ( g_1(\xib) , \dots, g_J(\xib))^\top\|_2 \\ 
    & \qquad \qquad + \sup_{\thetab \in \Theta, \xib \in \Xi}  \|
    \{h_j(y(\thetab;\xb_{t+1, j}),z_{t+1, j})  - \xi_{j}\}_{j=1}^K    
    \|_2 \\
    & \leq J C + 2KC   \leq 2 K J C L
\end{split}
\end{equation*}
by Assumption~\ref{asp:1} and $\|\lambdab^t\|_2 \leq \kappa \sqrt{K+J} \leq \kappa \sqrt{KJ}$.
Using Proposition~\ref{prop:noise-regret-bound} with $M = \kappa \sqrt{KJ}, W=2 K J C L$, we get 
\begin{equation*}
    \frac{1}{T} \sum_{t=1}^T  \Lcal(\thetab^t, \xib^t, \lambdab^t) - \min_{\lambdab \in \Lambda} \frac{1}{T} \sum_{t=1}^T \Lcal(\thetab^t, \xib^t,\lambdab) 
    = O\left( K^{5/4} J^{5/4} C L \sqrt{ \frac{\kappa \ln(1/\delta)}{T}}\right),
\end{equation*}
with probability $1-\delta$,
which implies \eqref{eq:OL-lambda} with
\begin{equation*}
    \epsilon_\lambda = O\left( \sqrt{ \frac{\kappa \ln(1/\delta)}{T}} \right),
\end{equation*}
by setting $\eta_{\lambda} = \sqrt{\kappa/(2 L C  K^{1/2} J^{1/2} T)}$.

The theorem follows by the union bound.

\bibliographystyle{imsart-number} 
\bibliography{ref}       

\begin{thebibliography}{71}

\bibitem{agarwal2018reductions}
\begin{binproceedings}[author]
\bauthor{\bsnm{Agarwal},~\bfnm{Alekh}\binits{A.}},
  \bauthor{\bsnm{Beygelzimer},~\bfnm{Alina}\binits{A.}},
  \bauthor{\bsnm{Dudik},~\bfnm{Miroslav}\binits{M.}},
  \bauthor{\bsnm{Langford},~\bfnm{John}\binits{J.}} \AND
  \bauthor{\bsnm{Wallach},~\bfnm{Hanna}\binits{H.}}
(\byear{2018}).
\btitle{A Reductions Approach to Fair Classification}.
In \bbooktitle{Proceedings of the 35th International Conference on Machine
  Learning}
(\beditor{\bfnm{Jennifer}\binits{J.}~\bsnm{Dy}} \AND
  \beditor{\bfnm{Andreas}\binits{A.}~\bsnm{Krause}}, eds.).
\bseries{Proceedings of Machine Learning Research}
\bvolume{80}
\bpages{60--69}.
\bpublisher{PMLR}.
\end{binproceedings}
\endbibitem

\bibitem{agrawal2014dynamic}
\begin{barticle}[author]
\bauthor{\bsnm{Agrawal},~\bfnm{Shipra}\binits{S.}},
  \bauthor{\bsnm{Wang},~\bfnm{Zizhuo}\binits{Z.}} \AND
  \bauthor{\bsnm{Ye},~\bfnm{Yinyu}\binits{Y.}}
(\byear{2014}).
\btitle{A Dynamic Near-Optimal Algorithm for Online Linear Programming}.
\bjournal{Operations Research}
\bvolume{62}
\bpages{876--890}.
\bdoi{10.1287/opre.2014.1289}
\end{barticle}
\endbibitem

\bibitem{alemohammad2020recurrent}
\begin{binproceedings}[author]
\bauthor{\bsnm{Alemohammad},~\bfnm{Sina}\binits{S.}},
  \bauthor{\bsnm{Wang},~\bfnm{Zichao}\binits{Z.}},
  \bauthor{\bsnm{Balestriero},~\bfnm{Randall}\binits{R.}} \AND
  \bauthor{\bsnm{Baraniuk},~\bfnm{Richard}\binits{R.}}
(\byear{2021}).
\btitle{The Recurrent Neural Tangent Kernel}.
In \bbooktitle{International Conference on Learning Representations}.
\end{binproceedings}
\endbibitem

\bibitem{allen2019What}
\begin{binproceedings}[author]
\bauthor{\bsnm{Allen-Zhu},~\bfnm{Zeyuan}\binits{Z.}} \AND
  \bauthor{\bsnm{Li},~\bfnm{Yuanzhi}\binits{Y.}}
(\byear{2019}).
\btitle{What Can ResNet Learn Efficiently, Going Beyond Kernels?}
In \bbooktitle{Advances in Neural Information Processing Systems}
(\beditor{\bfnm{H.}\binits{H.}~\bsnm{Wallach}},
  \beditor{\bfnm{H.}\binits{H.}~\bsnm{Larochelle}},
  \beditor{\bfnm{A.}\binits{A.}~\bsnm{Beygelzimer}},
  \beditor{\bfnm{F.}\binits{F.}~\bparticle{d\textquotesingle}
  \bsnm{Alch\'{e}-Buc}}, \beditor{\bfnm{E.}\binits{E.}~\bsnm{Fox}} \AND
  \beditor{\bfnm{R.}\binits{R.}~\bsnm{Garnett}}, eds.)
\bvolume{32}.
\bpublisher{Curran Associates, Inc.}
\end{binproceedings}
\endbibitem

\bibitem{allen2019learning}
\begin{binproceedings}[author]
\bauthor{\bsnm{Allen-Zhu},~\bfnm{Zeyuan}\binits{Z.}},
  \bauthor{\bsnm{Li},~\bfnm{Yuanzhi}\binits{Y.}} \AND
  \bauthor{\bsnm{Liang},~\bfnm{Yingyu}\binits{Y.}}
(\byear{2019}).
\btitle{Learning and Generalization in Overparameterized Neural Networks, Going
  Beyond Two Layers}.
In \bbooktitle{Advances in Neural Information Processing Systems}
(\beditor{\bfnm{H.}\binits{H.}~\bsnm{Wallach}},
  \beditor{\bfnm{H.}\binits{H.}~\bsnm{Larochelle}},
  \beditor{\bfnm{A.}\binits{A.}~\bsnm{Beygelzimer}},
  \beditor{\bfnm{F.}\binits{F.}~\bparticle{d\textquotesingle}
  \bsnm{Alch\'{e}-Buc}}, \beditor{\bfnm{E.}\binits{E.}~\bsnm{Fox}} \AND
  \beditor{\bfnm{R.}\binits{R.}~\bsnm{Garnett}}, eds.)
\bvolume{32}.
\bpublisher{Curran Associates, Inc.}
\end{binproceedings}
\endbibitem

\bibitem{allen2019convergence}
\begin{binproceedings}[author]
\bauthor{\bsnm{Allen-Zhu},~\bfnm{Zeyuan}\binits{Z.}},
  \bauthor{\bsnm{Li},~\bfnm{Yuanzhi}\binits{Y.}} \AND
  \bauthor{\bsnm{Song},~\bfnm{Zhao}\binits{Z.}}
(\byear{2019}).
\btitle{A Convergence Theory for Deep Learning via Over-Parameterization}.
In \bbooktitle{Proceedings of the 36th International Conference on Machine
  Learning}
(\beditor{\bfnm{Kamalika}\binits{K.}~\bsnm{Chaudhuri}} \AND
  \beditor{\bfnm{Ruslan}\binits{R.}~\bsnm{Salakhutdinov}}, eds.).
\bseries{Proceedings of Machine Learning Research}
\bvolume{97}
\bpages{242--252}.
\bpublisher{PMLR}.
\end{binproceedings}
\endbibitem

\bibitem{allen2019ConvergenceRate}
\begin{binproceedings}[author]
\bauthor{\bsnm{Allen-Zhu},~\bfnm{Zeyuan}\binits{Z.}},
  \bauthor{\bsnm{Li},~\bfnm{Yuanzhi}\binits{Y.}} \AND
  \bauthor{\bsnm{Song},~\bfnm{Zhao}\binits{Z.}}
(\byear{2019}).
\btitle{On the Convergence Rate of Training Recurrent Neural Networks}.
In \bbooktitle{Advances in Neural Information Processing Systems}
(\beditor{\bfnm{H.}\binits{H.}~\bsnm{Wallach}},
  \beditor{\bfnm{H.}\binits{H.}~\bsnm{Larochelle}},
  \beditor{\bfnm{A.}\binits{A.}~\bsnm{Beygelzimer}},
  \beditor{\bfnm{F.}\binits{F.}~\bparticle{d\textquotesingle}
  \bsnm{Alch\'{e}-Buc}}, \beditor{\bfnm{E.}\binits{E.}~\bsnm{Fox}} \AND
  \beditor{\bfnm{R.}\binits{R.}~\bsnm{Garnett}}, eds.)
\bvolume{32}.
\bpublisher{Curran Associates, Inc.}
\end{binproceedings}
\endbibitem

\bibitem{arora2019fine}
\begin{binproceedings}[author]
\bauthor{\bsnm{Arora},~\bfnm{Sanjeev}\binits{S.}},
  \bauthor{\bsnm{Du},~\bfnm{Simon}\binits{S.}},
  \bauthor{\bsnm{Hu},~\bfnm{Wei}\binits{W.}},
  \bauthor{\bsnm{Li},~\bfnm{Zhiyuan}\binits{Z.}} \AND
  \bauthor{\bsnm{Wang},~\bfnm{Ruosong}\binits{R.}}
(\byear{2019}).
\btitle{Fine-Grained Analysis of Optimization and Generalization for
  Overparameterized Two-Layer Neural Networks}.
In \bbooktitle{Proceedings of the 36th International Conference on Machine
  Learning}
(\beditor{\bfnm{Kamalika}\binits{K.}~\bsnm{Chaudhuri}} \AND
  \beditor{\bfnm{Ruslan}\binits{R.}~\bsnm{Salakhutdinov}}, eds.).
\bseries{Proceedings of Machine Learning Research}
\bvolume{97}
\bpages{322--332}.
\bpublisher{PMLR}.
\end{binproceedings}
\endbibitem

\bibitem{ba2016layer}
\begin{barticle}[author]
\bauthor{\bsnm{Ba},~\bfnm{Jimmy~Lei}\binits{J.~L.}},
  \bauthor{\bsnm{Kiros},~\bfnm{Jamie~Ryan}\binits{J.~R.}} \AND
  \bauthor{\bsnm{Hinton},~\bfnm{Geoffrey~E}\binits{G.~E.}}
(\byear{2016}).
\btitle{Layer normalization}.
\bjournal{arXiv preprint arXiv:1607.06450}.
\end{barticle}
\endbibitem

\bibitem{bauschke1997legendre}
\begin{barticle}[author]
\bauthor{\bsnm{Bauschke},~\bfnm{Heinz~H}\binits{H.~H.}},
  \bauthor{\bsnm{Borwein},~\bfnm{Jonathan~M}\binits{J.~M.}} \betal{et~al.}
(\byear{1997}).
\btitle{Legendre functions and the method of random Bregman projections}.
\bjournal{Journal of convex analysis}
\bvolume{4}
\bpages{27--67}.
\end{barticle}
\endbibitem

\bibitem{bertsekas2014constrained}
\begin{bbook}[author]
\bauthor{\bsnm{Bertsekas},~\bfnm{Dimitri~P}\binits{D.~P.}}
(\byear{2014}).
\btitle{Constrained optimization and Lagrange multiplier methods}.
\bpublisher{Academic press}.
\end{bbook}
\endbibitem

\bibitem{blum2019advancing}
\begin{barticle}[author]
\bauthor{\bsnm{Blum},~\bfnm{Avrim}\binits{A.}} \AND
  \bauthor{\bsnm{Lykouris},~\bfnm{Thodoris}\binits{T.}}
\btitle{Advancing Subgroup Fairness via Sleeping Experts}.
\bjournal{Innovations in Theoretical Computer Science Conference (ITCS)}
\bvolume{11}.
\end{barticle}
\endbibitem

\bibitem{blum2019recovering}
\begin{barticle}[author]
\bauthor{\bsnm{Blum},~\bfnm{Avrim}\binits{A.}} \AND
  \bauthor{\bsnm{Stangl},~\bfnm{Kevin}\binits{K.}}
\btitle{Recovering from Biased Data: Can Fairness Constraints Improve
  Accuracy?}
\bjournal{Symposium on Foundations of Responsible Computing (FORC)}
\bvolume{1}.
\end{barticle}
\endbibitem

\bibitem{boob2019stochastic}
\begin{barticle}[author]
\bauthor{\bsnm{Boob},~\bfnm{Digvijay}\binits{D.}},
  \bauthor{\bsnm{Deng},~\bfnm{Qi}\binits{Q.}} \AND
  \bauthor{\bsnm{Lan},~\bfnm{Guanghui}\binits{G.}}
(\byear{2022}).
\btitle{Stochastic first-order methods for convex and nonconvex functional
  constrained optimization}.
\bjournal{Mathematical Programming}.
\bdoi{10.1007/s10107-021-01742-y}
\end{barticle}
\endbibitem

\bibitem{cai2019neural}
\begin{binproceedings}[author]
\bauthor{\bsnm{Cai},~\bfnm{Qi}\binits{Q.}},
  \bauthor{\bsnm{Yang},~\bfnm{Zhuoran}\binits{Z.}},
  \bauthor{\bsnm{Lee},~\bfnm{Jason~D}\binits{J.~D.}} \AND
  \bauthor{\bsnm{Wang},~\bfnm{Zhaoran}\binits{Z.}}
(\byear{2019}).
\btitle{Neural Temporal-Difference Learning Converges to Global Optima}.
In \bbooktitle{Advances in Neural Information Processing Systems}
(\beditor{\bfnm{H.}\binits{H.}~\bsnm{Wallach}},
  \beditor{\bfnm{H.}\binits{H.}~\bsnm{Larochelle}},
  \beditor{\bfnm{A.}\binits{A.}~\bsnm{Beygelzimer}},
  \beditor{\bfnm{F.}\binits{F.}~\bparticle{d\textquotesingle}
  \bsnm{Alch\'{e}-Buc}}, \beditor{\bfnm{E.}\binits{E.}~\bsnm{Fox}} \AND
  \beditor{\bfnm{R.}\binits{R.}~\bsnm{Garnett}}, eds.)
\bvolume{32}.
\bpublisher{Curran Associates, Inc.}
\end{binproceedings}
\endbibitem

\bibitem{cartis2017corrigendum}
\begin{barticle}[author]
\bauthor{\bsnm{Cartis},~\bfnm{C.}\binits{C.}},
  \bauthor{\bsnm{Gould},~\bfnm{N.~I.~M.}\binits{N.~I.~M.}} \AND
  \bauthor{\bsnm{Toint},~\bfnm{Ph.~L.}\binits{P.~L.}}
(\byear{2016}).
\btitle{Corrigendum: On the complexity of finding first-order critical points
  in constrained nonlinear optimization}.
\bjournal{Mathematical Programming}
\bvolume{161}
\bpages{611--626}.
\bdoi{10.1007/s10107-016-1016-4}
\end{barticle}
\endbibitem

\bibitem{celis2019classification}
\begin{binproceedings}[author]
\bauthor{\bsnm{Celis},~\bfnm{L.~Elisa}\binits{L.~E.}},
  \bauthor{\bsnm{Huang},~\bfnm{Lingxiao}\binits{L.}},
  \bauthor{\bsnm{Keswani},~\bfnm{Vijay}\binits{V.}} \AND
  \bauthor{\bsnm{Vishnoi},~\bfnm{Nisheeth~K.}\binits{N.~K.}}
(\byear{2019}).
\btitle{Classification with Fairness Constraints}.
In \bbooktitle{Proceedings of the Conference on Fairness, Accountability, and
  Transparency}
\bpages{319--328}.
\bpublisher{{ACM}}.
\bdoi{10.1145/3287560.3287586}
\end{binproceedings}
\endbibitem

\bibitem{chen2017robust}
\begin{binproceedings}[author]
\bauthor{\bsnm{Chen},~\bfnm{Robert~S.}\binits{R.~S.}},
  \bauthor{\bsnm{Lucier},~\bfnm{Brendan}\binits{B.}},
  \bauthor{\bsnm{Singer},~\bfnm{Yaron}\binits{Y.}} \AND
  \bauthor{\bsnm{Syrgkanis},~\bfnm{Vasilis}\binits{V.}}
(\byear{2017}).
\btitle{Robust Optimization for Non-Convex Objectives}.
In \bbooktitle{Advances in Neural Information Processing Systems}
(\beditor{\bfnm{I.}\binits{I.}~\bsnm{Guyon}},
  \beditor{\bfnm{U.~Von}\binits{U.~V.}~\bsnm{Luxburg}},
  \beditor{\bfnm{S.}\binits{S.}~\bsnm{Bengio}},
  \beditor{\bfnm{H.}\binits{H.}~\bsnm{Wallach}},
  \beditor{\bfnm{R.}\binits{R.}~\bsnm{Fergus}},
  \beditor{\bfnm{S.}\binits{S.}~\bsnm{Vishwanathan}} \AND
  \beditor{\bfnm{R.}\binits{R.}~\bsnm{Garnett}}, eds.)
\bvolume{30}.
\bpublisher{Curran Associates, Inc.}
\end{binproceedings}
\endbibitem

\bibitem{chen2021theorem}
\begin{barticle}[author]
\bauthor{\bsnm{Chen},~\bfnm{Shuxiao}\binits{S.}},
  \bauthor{\bsnm{Zheng},~\bfnm{Qinqing}\binits{Q.}},
  \bauthor{\bsnm{Long},~\bfnm{Qi}\binits{Q.}} \AND
  \bauthor{\bsnm{Su},~\bfnm{Weijie~J.}\binits{W.~J.}}
(\byear{2021}).
\btitle{A Theorem of the Alternative for Personalized Federated Learning}.
\bjournal{CoRR}
\bvolume{abs/2103.01901}.
\end{barticle}
\endbibitem

\bibitem{chen2021tensor}
\begin{barticle}[author]
\bauthor{\bsnm{Chen},~\bfnm{You-Lin}\binits{Y.-L.}},
  \bauthor{\bsnm{Kolar},~\bfnm{Mladen}\binits{M.}} \AND
  \bauthor{\bsnm{Tsay},~\bfnm{Ruey~S.}\binits{R.~S.}}
(\byear{2021}).
\btitle{Tensor Canonical Correlation Analysis With Convergence and Statistical
  Guarantees}.
\bjournal{Journal of Computational and Graphical Statistics}
\bvolume{30}
\bpages{728--744}.
\bdoi{10.1080/10618600.2020.1856118}
\end{barticle}
\endbibitem

\bibitem{chenonline}
\begin{barticle}[author]
\bauthor{\bsnm{Chen},~\bfnm{Zhehui}\binits{Z.}},
  \bauthor{\bsnm{Li},~\bfnm{Xingguo}\binits{X.}},
  \bauthor{\bsnm{Yang},~\bfnm{Lin}\binits{L.}},
  \bauthor{\bsnm{Haupt},~\bfnm{Jarvis}\binits{J.}} \AND
  \bauthor{\bsnm{Zhao},~\bfnm{Tuo}\binits{T.}}
(\byear{2017}).
\btitle{Online Generalized Eigenvalue Decomposition: Primal Dual Geometry and
  Inverse-Free Stochastic Optimization}.
\end{barticle}
\endbibitem

\bibitem{chizat2019lazy}
\begin{binproceedings}[author]
\bauthor{\bsnm{Chizat},~\bfnm{L\'{e}na\"{\i}c}\binits{L.}},
  \bauthor{\bsnm{Oyallon},~\bfnm{Edouard}\binits{E.}} \AND
  \bauthor{\bsnm{Bach},~\bfnm{Francis}\binits{F.}}
(\byear{2019}).
\btitle{On Lazy Training in Differentiable Programming}.
In \bbooktitle{Advances in Neural Information Processing Systems}
(\beditor{\bfnm{H.}\binits{H.}~\bsnm{Wallach}},
  \beditor{\bfnm{H.}\binits{H.}~\bsnm{Larochelle}},
  \beditor{\bfnm{A.}\binits{A.}~\bsnm{Beygelzimer}},
  \beditor{\bfnm{F.}\binits{F.}~\bparticle{d\textquotesingle}
  \bsnm{Alch\'{e}-Buc}}, \beditor{\bfnm{E.}\binits{E.}~\bsnm{Fox}} \AND
  \beditor{\bfnm{R.}\binits{R.}~\bsnm{Garnett}}, eds.)
\bvolume{32}.
\bpublisher{Curran Associates, Inc.}
\end{binproceedings}
\endbibitem

\bibitem{chouldechova2017fair}
\begin{barticle}[author]
\bauthor{\bsnm{Chouldechova},~\bfnm{Alexandra}\binits{A.}}
(\byear{2017}).
\btitle{Fair Prediction with Disparate Impact: A Study of Bias in Recidivism
  Prediction Instruments}.
\bjournal{Big Data}
\bvolume{5}
\bpages{153--163}.
\bdoi{10.1089/big.2016.0047}
\end{barticle}
\endbibitem

\bibitem{chow2003probability}
\begin{bbook}[author]
\bauthor{\bsnm{Chow},~\bfnm{Yuan~Shih}\binits{Y.~S.}} \AND
  \bauthor{\bsnm{Teicher},~\bfnm{Henry}\binits{H.}}
(\byear{2003}).
\btitle{Probability theory: independence, interchangeability, martingales}.
\bpublisher{Springer Science \& Business Media}.
\end{bbook}
\endbibitem

\bibitem{cotter2019making}
\begin{binproceedings}[author]
\bauthor{\bsnm{Cotter},~\bfnm{Andrew}\binits{A.}},
  \bauthor{\bsnm{Gupta},~\bfnm{Maya}\binits{M.}} \AND
  \bauthor{\bsnm{Narasimhan},~\bfnm{Harikrishna}\binits{H.}}
(\byear{2019}).
\btitle{On Making Stochastic Classifiers Deterministic}.
In \bbooktitle{Advances in Neural Information Processing Systems}
(\beditor{\bfnm{H.}\binits{H.}~\bsnm{Wallach}},
  \beditor{\bfnm{H.}\binits{H.}~\bsnm{Larochelle}},
  \beditor{\bfnm{A.}\binits{A.}~\bsnm{Beygelzimer}},
  \beditor{\bfnm{F.}\binits{F.}~\bparticle{d\textquotesingle}
  \bsnm{Alch\'{e}-Buc}}, \beditor{\bfnm{E.}\binits{E.}~\bsnm{Fox}} \AND
  \beditor{\bfnm{R.}\binits{R.}~\bsnm{Garnett}}, eds.)
\bvolume{32}.
\bpublisher{Curran Associates, Inc.}
\end{binproceedings}
\endbibitem

\bibitem{cotter2019optimization}
\begin{barticle}[author]
\bauthor{\bsnm{Cotter},~\bfnm{Andrew}\binits{A.}},
  \bauthor{\bsnm{Jiang},~\bfnm{Heinrich}\binits{H.}},
  \bauthor{\bsnm{Gupta},~\bfnm{Maya}\binits{M.}},
  \bauthor{\bsnm{Wang},~\bfnm{Serena}\binits{S.}},
  \bauthor{\bsnm{Narayan},~\bfnm{Taman}\binits{T.}},
  \bauthor{\bsnm{You},~\bfnm{Seungil}\binits{S.}} \AND
  \bauthor{\bsnm{Sridharan},~\bfnm{Karthik}\binits{K.}}
(\byear{2019}).
\btitle{Optimization with Non-Differentiable Constraints with Applications to
  Fairness, Recall, Churn, and Other Goals}.
\bjournal{Journal of Machine Learning Research}
\bvolume{20}
\bpages{1--59}.
\end{barticle}
\endbibitem

\bibitem{cotter2019two}
\begin{binproceedings}[author]
\bauthor{\bsnm{Cotter},~\bfnm{Andrew}\binits{A.}},
  \bauthor{\bsnm{Jiang},~\bfnm{Heinrich}\binits{H.}} \AND
  \bauthor{\bsnm{Sridharan},~\bfnm{Karthik}\binits{K.}}
(\byear{2019}).
\btitle{Two-Player Games for Efficient Non-Convex Constrained Optimization}.
In \bbooktitle{Proceedings of the 30th International Conference on Algorithmic
  Learning Theory}
(\beditor{\bfnm{Aurélien}\binits{A.}~\bsnm{Garivier}} \AND
  \beditor{\bfnm{Satyen}\binits{S.}~\bsnm{Kale}}, eds.).
\bseries{Proceedings of Machine Learning Research}
\bvolume{98}
\bpages{300--332}.
\bpublisher{PMLR}.
\end{binproceedings}
\endbibitem

\bibitem{daskalaki2006evaluation}
\begin{barticle}[author]
\bauthor{\bsnm{Daskalaki},~\bfnm{Sophia}\binits{S.}},
  \bauthor{\bsnm{Kopanas},~\bfnm{Ioannis}\binits{I.}} \AND
  \bauthor{\bsnm{Avouris},~\bfnm{Nikolaos}\binits{N.}}
(\byear{2006}).
\btitle{Evaluation of Classifiers for an Uneven Class Distribution Problem}.
\bjournal{Applied Artificial Intelligence}
\bvolume{20}
\bpages{381--417}.
\bdoi{10.1080/08839510500313653}
\end{barticle}
\endbibitem

\bibitem{davis2019stochastic}
\begin{barticle}[author]
\bauthor{\bsnm{Davis},~\bfnm{Damek}\binits{D.}} \AND
  \bauthor{\bsnm{Drusvyatskiy},~\bfnm{Dmitriy}\binits{D.}}
(\byear{2019}).
\btitle{Stochastic model-based minimization of weakly convex functions}.
\bjournal{SIAM Journal on Optimization}
\bvolume{29}
\bpages{207--239}.
\bdoi{10.1137/18M1178244}
\end{barticle}
\endbibitem

\bibitem{denevi2019learning}
\begin{binproceedings}[author]
\bauthor{\bsnm{Denevi},~\bfnm{Giulia}\binits{G.}},
  \bauthor{\bsnm{Ciliberto},~\bfnm{Carlo}\binits{C.}},
  \bauthor{\bsnm{Grazzi},~\bfnm{Riccardo}\binits{R.}} \AND
  \bauthor{\bsnm{Pontil},~\bfnm{Massimiliano}\binits{M.}}
(\byear{2019}).
\btitle{Learning-to-Learn Stochastic Gradient Descent with Biased
  Regularization}.
In \bbooktitle{Proceedings of the 36th International Conference on Machine
  Learning}
(\beditor{\bfnm{Kamalika}\binits{K.}~\bsnm{Chaudhuri}} \AND
  \beditor{\bfnm{Ruslan}\binits{R.}~\bsnm{Salakhutdinov}}, eds.).
\bseries{Proceedings of Machine Learning Research}
\bvolume{97}
\bpages{1566--1575}.
\bpublisher{PMLR}.
\end{binproceedings}
\endbibitem

\bibitem{donini2018empirical}
\begin{binproceedings}[author]
\bauthor{\bsnm{Donini},~\bfnm{Michele}\binits{M.}},
  \bauthor{\bsnm{Oneto},~\bfnm{Luca}\binits{L.}},
  \bauthor{\bsnm{Ben-David},~\bfnm{Shai}\binits{S.}},
  \bauthor{\bsnm{Shawe-Taylor},~\bfnm{John~S}\binits{J.~S.}} \AND
  \bauthor{\bsnm{Pontil},~\bfnm{Massimiliano}\binits{M.}}
(\byear{2018}).
\btitle{Empirical Risk Minimization Under Fairness Constraints}.
In \bbooktitle{Advances in Neural Information Processing Systems}
(\beditor{\bfnm{S.}\binits{S.}~\bsnm{Bengio}},
  \beditor{\bfnm{H.}\binits{H.}~\bsnm{Wallach}},
  \beditor{\bfnm{H.}\binits{H.}~\bsnm{Larochelle}},
  \beditor{\bfnm{K.}\binits{K.}~\bsnm{Grauman}},
  \beditor{\bfnm{N.}\binits{N.}~\bsnm{Cesa-Bianchi}} \AND
  \beditor{\bfnm{R.}\binits{R.}~\bsnm{Garnett}}, eds.)
\bvolume{31}.
\bpublisher{Curran Associates, Inc.}
\end{binproceedings}
\endbibitem

\bibitem{dressel2018accuracy}
\begin{barticle}[author]
\bauthor{\bsnm{Dressel},~\bfnm{Julia}\binits{J.}} \AND
  \bauthor{\bsnm{Farid},~\bfnm{Hany}\binits{H.}}
(\byear{2018}).
\btitle{The accuracy, fairness, and limits of predicting recidivism}.
\bjournal{Science Advances}
\bvolume{4}
\bpages{eaao5580}.
\bdoi{10.1126/sciadv.aao5580}
\end{barticle}
\endbibitem

\bibitem{du2019gradient}
\begin{binproceedings}[author]
\bauthor{\bsnm{Du},~\bfnm{Simon}\binits{S.}},
  \bauthor{\bsnm{Lee},~\bfnm{Jason}\binits{J.}},
  \bauthor{\bsnm{Li},~\bfnm{Haochuan}\binits{H.}},
  \bauthor{\bsnm{Wang},~\bfnm{Liwei}\binits{L.}} \AND
  \bauthor{\bsnm{Zhai},~\bfnm{Xiyu}\binits{X.}}
(\byear{2019}).
\btitle{Gradient Descent Finds Global Minima of Deep Neural Networks}.
In \bbooktitle{Proceedings of the 36th International Conference on Machine
  Learning}
(\beditor{\bfnm{Kamalika}\binits{K.}~\bsnm{Chaudhuri}} \AND
  \beditor{\bfnm{Ruslan}\binits{R.}~\bsnm{Salakhutdinov}}, eds.).
\bseries{Proceedings of Machine Learning Research}
\bvolume{97}
\bpages{1675--1685}.
\bpublisher{PMLR}.
\end{binproceedings}
\endbibitem

\bibitem{dwork2012fairness}
\begin{binproceedings}[author]
\bauthor{\bsnm{Dwork},~\bfnm{Cynthia}\binits{C.}},
  \bauthor{\bsnm{Hardt},~\bfnm{Moritz}\binits{M.}},
  \bauthor{\bsnm{Pitassi},~\bfnm{Toniann}\binits{T.}},
  \bauthor{\bsnm{Reingold},~\bfnm{Omer}\binits{O.}} \AND
  \bauthor{\bsnm{Zemel},~\bfnm{Richard}\binits{R.}}
(\byear{2012}).
\btitle{Fairness through awareness}.
In \bbooktitle{Proceedings of the 3rd Innovations in Theoretical Computer
  Science Conference on - {ITCS} {\textquotesingle}12}
\bpages{214--226}.
\bpublisher{{ACM} Press}.
\bdoi{10.1145/2090236.2090255}
\end{binproceedings}
\endbibitem

\bibitem{esuli2015optimizing}
\begin{barticle}[author]
\bauthor{\bsnm{Esuli},~\bfnm{Andrea}\binits{A.}} \AND
  \bauthor{\bsnm{Sebastiani},~\bfnm{Fabrizio}\binits{F.}}
(\byear{2015}).
\btitle{Optimizing text quantifiers for multivariate loss functions}.
\bjournal{ACM Transactions on Knowledge Discovery from Data (TKDD)}
\bvolume{9}
\bpages{1--27}.
\bdoi{10.1145/2700406}
\end{barticle}
\endbibitem

\bibitem{fan2019selective}
\begin{barticle}[author]
\bauthor{\bsnm{Fan},~\bfnm{Jianqing}\binits{J.}},
  \bauthor{\bsnm{Ma},~\bfnm{Cong}\binits{C.}} \AND
  \bauthor{\bsnm{Zhong},~\bfnm{Yiqiao}\binits{Y.}}
(\byear{2021}).
\btitle{A Selective Overview of Deep Learning}.
\bjournal{Statistical Science}
\bvolume{36}.
\bdoi{10.1214/20-sts783}
\end{barticle}
\endbibitem

\bibitem{feldman2012agnostic}
\begin{barticle}[author]
\bauthor{\bsnm{Feldman},~\bfnm{Vitaly}\binits{V.}},
  \bauthor{\bsnm{Guruswami},~\bfnm{Venkatesan}\binits{V.}},
  \bauthor{\bsnm{Raghavendra},~\bfnm{Prasad}\binits{P.}} \AND
  \bauthor{\bsnm{Wu},~\bfnm{Yi}\binits{Y.}}
(\byear{2012}).
\btitle{Agnostic Learning of Monomials by Halfspaces Is Hard}.
\bjournal{{SIAM} Journal on Computing}
\bvolume{41}
\bpages{1558--1590}.
\bdoi{10.1137/120865094}
\end{barticle}
\endbibitem

\bibitem{gao2019convergence}
\begin{binproceedings}[author]
\bauthor{\bsnm{Gao},~\bfnm{Ruiqi}\binits{R.}},
  \bauthor{\bsnm{Cai},~\bfnm{Tianle}\binits{T.}},
  \bauthor{\bsnm{Li},~\bfnm{Haochuan}\binits{H.}},
  \bauthor{\bsnm{Hsieh},~\bfnm{Cho-Jui}\binits{C.-J.}},
  \bauthor{\bsnm{Wang},~\bfnm{Liwei}\binits{L.}} \AND
  \bauthor{\bsnm{Lee},~\bfnm{Jason~D}\binits{J.~D.}}
(\byear{2019}).
\btitle{Convergence of Adversarial Training in Overparametrized Neural
  Networks}.
In \bbooktitle{Advances in Neural Information Processing Systems}
(\beditor{\bfnm{H.}\binits{H.}~\bsnm{Wallach}},
  \beditor{\bfnm{H.}\binits{H.}~\bsnm{Larochelle}},
  \beditor{\bfnm{A.}\binits{A.}~\bsnm{Beygelzimer}},
  \beditor{\bfnm{F.}\binits{F.}~\bparticle{d\textquotesingle}
  \bsnm{Alch\'{e}-Buc}}, \beditor{\bfnm{E.}\binits{E.}~\bsnm{Fox}} \AND
  \beditor{\bfnm{R.}\binits{R.}~\bsnm{Garnett}}, eds.)
\bvolume{32}.
\bpublisher{Curran Associates, Inc.}
\end{binproceedings}
\endbibitem

\bibitem{gao2015tweet}
\begin{binproceedings}[author]
\bauthor{\bsnm{Gao},~\bfnm{Wei}\binits{W.}} \AND
  \bauthor{\bsnm{Sebastiani},~\bfnm{Fabrizio}\binits{F.}}
(\byear{2015}).
\btitle{Tweet sentiment: From classification to quantification}.
In \bbooktitle{2015 IEEE/ACM International Conference on Advances in Social
  Networks Analysis and Mining (ASONAM)}
\bpages{97-104}.
\bdoi{10.1145/2808797.2809327}
\end{binproceedings}
\endbibitem

\bibitem{hardt2016equality}
\begin{binproceedings}[author]
\bauthor{\bsnm{Hardt},~\bfnm{Moritz}\binits{M.}},
  \bauthor{\bsnm{Price},~\bfnm{Eric}\binits{E.}},
  \bauthor{\bsnm{Price},~\bfnm{Eric}\binits{E.}} \AND
  \bauthor{\bsnm{Srebro},~\bfnm{Nati}\binits{N.}}
(\byear{2016}).
\btitle{Equality of Opportunity in Supervised Learning}.
In \bbooktitle{Advances in Neural Information Processing Systems}
(\beditor{\bfnm{D.}\binits{D.}~\bsnm{Lee}},
  \beditor{\bfnm{M.}\binits{M.}~\bsnm{Sugiyama}},
  \beditor{\bfnm{U.}\binits{U.}~\bsnm{Luxburg}},
  \beditor{\bfnm{I.}\binits{I.}~\bsnm{Guyon}} \AND
  \beditor{\bfnm{R.}\binits{R.}~\bsnm{Garnett}}, eds.)
\bvolume{29}.
\bpublisher{Curran Associates, Inc.}
\end{binproceedings}
\endbibitem

\bibitem{hastie2009elements}
\begin{bbook}[author]
\bauthor{\bsnm{Hastie},~\bfnm{Trevor}\binits{T.}},
  \bauthor{\bsnm{Tibshirani},~\bfnm{Robert}\binits{R.}} \AND
  \bauthor{\bsnm{Friedman},~\bfnm{Jerome}\binits{J.}}
(\byear{2009}).
\btitle{The Elements of Statistical Learning}.
\bseries{Springer Series in Statistics}.
\bpublisher{Springer New York}.
\bdoi{10.1007/978-0-387-84858-7}
\end{bbook}
\endbibitem

\bibitem{huang2017following}
\begin{barticle}[author]
\bauthor{\bsnm{Huang},~\bfnm{Ruitong}\binits{R.}},
  \bauthor{\bsnm{Lattimore},~\bfnm{Tor}\binits{T.}},
  \bauthor{\bsnm{Gy{{\"o}}rgy},~\bfnm{Andr{{\'a}}s}\binits{A.}} \AND
  \bauthor{\bsnm{Szepesv{{\'a}}ri},~\bfnm{Csaba}\binits{C.}}
(\byear{2017}).
\btitle{Following the Leader and Fast Rates in Online Linear Prediction: Curved
  Constraint Sets and Other Regularities}.
\bjournal{Journal of Machine Learning Research}
\bvolume{18}
\bpages{1--31}.
\end{barticle}
\endbibitem

\bibitem{jacot2018neural}
\begin{binproceedings}[author]
\bauthor{\bsnm{Jacot},~\bfnm{Arthur}\binits{A.}},
  \bauthor{\bsnm{Gabriel},~\bfnm{Franck}\binits{F.}} \AND
  \bauthor{\bsnm{Hongler},~\bfnm{Clement}\binits{C.}}
(\byear{2018}).
\btitle{Neural Tangent Kernel: Convergence and Generalization in Neural
  Networks}.
In \bbooktitle{Advances in Neural Information Processing Systems}
(\beditor{\bfnm{S.}\binits{S.}~\bsnm{Bengio}},
  \beditor{\bfnm{H.}\binits{H.}~\bsnm{Wallach}},
  \beditor{\bfnm{H.}\binits{H.}~\bsnm{Larochelle}},
  \beditor{\bfnm{K.}\binits{K.}~\bsnm{Grauman}},
  \beditor{\bfnm{N.}\binits{N.}~\bsnm{Cesa-Bianchi}} \AND
  \beditor{\bfnm{R.}\binits{R.}~\bsnm{Garnett}}, eds.)
\bvolume{31}.
\bpublisher{Curran Associates, Inc.}
\end{binproceedings}
\endbibitem

\bibitem{jain2017non}
\begin{barticle}[author]
\bauthor{\bsnm{Jain},~\bfnm{Prateek}\binits{P.}} \AND
  \bauthor{\bsnm{Kar},~\bfnm{Purushottam}\binits{P.}}
(\byear{2017}).
\btitle{Non-convex Optimization for Machine Learning}.
\bjournal{Foundations and Trends{\textregistered} in Machine Learning}
\bvolume{10}
\bpages{142--336}.
\end{barticle}
\endbibitem

\bibitem{kennedy2009learning}
\begin{binproceedings}[author]
\bauthor{\bsnm{Kennedy},~\bfnm{Kenneth}\binits{K.}},
  \bauthor{\bsnm{Namee},~\bfnm{Brian~Mac}\binits{B.~M.}} \AND
  \bauthor{\bsnm{Delany},~\bfnm{Sarah~Jane}\binits{S.~J.}}
(\byear{2010}).
\btitle{Learning without Default: A Study of One-Class Classification and the
  Low-Default Portfolio Problem}.
In \bbooktitle{Artificial Intelligence and Cognitive Science}
\bpages{174--187}.
\borganization{Springer}.
\bpublisher{Springer Berlin Heidelberg}.
\bdoi{10.1007/978-3-642-17080-5_20}
\end{binproceedings}
\endbibitem

\bibitem{kilbertus2017avoiding}
\begin{binproceedings}[author]
\bauthor{\bsnm{Kilbertus},~\bfnm{Niki}\binits{N.}},
  \bauthor{\bsnm{Rojas~Carulla},~\bfnm{Mateo}\binits{M.}},
  \bauthor{\bsnm{Parascandolo},~\bfnm{Giambattista}\binits{G.}},
  \bauthor{\bsnm{Hardt},~\bfnm{Moritz}\binits{M.}},
  \bauthor{\bsnm{Janzing},~\bfnm{Dominik}\binits{D.}} \AND
  \bauthor{\bsnm{Sch\"{o}lkopf},~\bfnm{Bernhard}\binits{B.}}
(\byear{2017}).
\btitle{Avoiding Discrimination through Causal Reasoning}.
In \bbooktitle{Advances in Neural Information Processing Systems}
(\beditor{\bfnm{I.}\binits{I.}~\bsnm{Guyon}},
  \beditor{\bfnm{U.~Von}\binits{U.~V.}~\bsnm{Luxburg}},
  \beditor{\bfnm{S.}\binits{S.}~\bsnm{Bengio}},
  \beditor{\bfnm{H.}\binits{H.}~\bsnm{Wallach}},
  \beditor{\bfnm{R.}\binits{R.}~\bsnm{Fergus}},
  \beditor{\bfnm{S.}\binits{S.}~\bsnm{Vishwanathan}} \AND
  \beditor{\bfnm{R.}\binits{R.}~\bsnm{Garnett}}, eds.)
\bvolume{30}.
\bpublisher{Curran Associates, Inc.}
\end{binproceedings}
\endbibitem

\bibitem{komiyama2018nonconvex}
\begin{binproceedings}[author]
\bauthor{\bsnm{Komiyama},~\bfnm{Junpei}\binits{J.}},
  \bauthor{\bsnm{Takeda},~\bfnm{Akiko}\binits{A.}},
  \bauthor{\bsnm{Honda},~\bfnm{Junya}\binits{J.}} \AND
  \bauthor{\bsnm{Shimao},~\bfnm{Hajime}\binits{H.}}
(\byear{2018}).
\btitle{Nonconvex Optimization for Regression with Fairness Constraints}.
In \bbooktitle{Proceedings of the 35th International Conference on Machine
  Learning}
(\beditor{\bfnm{Jennifer}\binits{J.}~\bsnm{Dy}} \AND
  \beditor{\bfnm{Andreas}\binits{A.}~\bsnm{Krause}}, eds.).
\bseries{Proceedings of Machine Learning Research}
\bvolume{80}
\bpages{2737--2746}.
\bpublisher{PMLR}.
\end{binproceedings}
\endbibitem

\bibitem{krogh1991simple}
\begin{binproceedings}[author]
\bauthor{\bsnm{Krogh},~\bfnm{Anders}\binits{A.}} \AND
  \bauthor{\bsnm{Hertz},~\bfnm{John}\binits{J.}}
(\byear{1991}).
\btitle{A Simple Weight Decay Can Improve Generalization}.
In \bbooktitle{Advances in Neural Information Processing Systems}
(\beditor{\bfnm{J.}\binits{J.}~\bsnm{Moody}},
  \beditor{\bfnm{S.}\binits{S.}~\bsnm{Hanson}} \AND
  \beditor{\bfnm{R.~P.}\binits{R.~P.}~\bsnm{Lippmann}}, eds.)
\bvolume{4}.
\bpublisher{Morgan-Kaufmann}.
\end{binproceedings}
\endbibitem

\bibitem{kubat1997addressing}
\begin{binproceedings}[author]
\bauthor{\bsnm{Kubat},~\bfnm{Miroslav}\binits{M.}} \AND
  \bauthor{\bsnm{Matwin},~\bfnm{Stan}\binits{S.}}
(\byear{1997}).
\btitle{Addressing the Curse of Imbalanced Training Sets: One-Sided Selection}.
In \bbooktitle{In Proceedings of the Fourteenth International Conference on
  Machine Learning}
\bpages{179--186}.
\bpublisher{Morgan Kaufmann}.
\end{binproceedings}
\endbibitem

\bibitem{lawrence2012neural}
\begin{bincollection}[author]
\bauthor{\bsnm{Lawrence},~\bfnm{Steve}\binits{S.}},
  \bauthor{\bsnm{Burns},~\bfnm{Ian}\binits{I.}},
  \bauthor{\bsnm{Back},~\bfnm{Andrew}\binits{A.}},
  \bauthor{\bsnm{Tsoi},~\bfnm{Ah~Chung}\binits{A.~C.}} \AND
  \bauthor{\bsnm{Giles},~\bfnm{C.~Lee}\binits{C.~L.}}
(\byear{2012}).
\btitle{Neural Network Classification and Prior Class Probabilities}.
In \bbooktitle{Lecture Notes in Computer Science}
\bpages{295--309}.
\bpublisher{Springer Berlin Heidelberg}.
\bdoi{10.1007/978-3-642-35289-8_19}
\end{bincollection}
\endbibitem

\bibitem{lee2019wide}
\begin{binproceedings}[author]
\bauthor{\bsnm{Lee},~\bfnm{Jaehoon}\binits{J.}},
  \bauthor{\bsnm{Xiao},~\bfnm{Lechao}\binits{L.}},
  \bauthor{\bsnm{Schoenholz},~\bfnm{Samuel}\binits{S.}},
  \bauthor{\bsnm{Bahri},~\bfnm{Yasaman}\binits{Y.}},
  \bauthor{\bsnm{Novak},~\bfnm{Roman}\binits{R.}},
  \bauthor{\bsnm{Sohl-Dickstein},~\bfnm{Jascha}\binits{J.}} \AND
  \bauthor{\bsnm{Pennington},~\bfnm{Jeffrey}\binits{J.}}
(\byear{2019}).
\btitle{Wide Neural Networks of Any Depth Evolve as Linear Models Under
  Gradient Descent}.
In \bbooktitle{Advances in Neural Information Processing Systems}
(\beditor{\bfnm{H.}\binits{H.}~\bsnm{Wallach}},
  \beditor{\bfnm{H.}\binits{H.}~\bsnm{Larochelle}},
  \beditor{\bfnm{A.}\binits{A.}~\bsnm{Beygelzimer}},
  \beditor{\bfnm{F.}\binits{F.}~\bparticle{d\textquotesingle}
  \bsnm{Alch\'{e}-Buc}}, \beditor{\bfnm{E.}\binits{E.}~\bsnm{Fox}} \AND
  \beditor{\bfnm{R.}\binits{R.}~\bsnm{Garnett}}, eds.)
\bvolume{32}.
\bpublisher{Curran Associates, Inc.}
\end{binproceedings}
\endbibitem

\bibitem{li2019online}
\begin{barticle}[author]
\bauthor{\bsnm{Li},~\bfnm{Xiaocheng}\binits{X.}} \AND
  \bauthor{\bsnm{Ye},~\bfnm{Yinyu}\binits{Y.}}
(\byear{2021}).
\btitle{Online Linear Programming: Dual Convergence, New Algorithms, and Regret
  Bounds}.
\bjournal{Operations Research}.
\bdoi{10.1287/opre.2021.2164}
\end{barticle}
\endbibitem

\bibitem{li2018learning}
\begin{binproceedings}[author]
\bauthor{\bsnm{Li},~\bfnm{Yuanzhi}\binits{Y.}} \AND
  \bauthor{\bsnm{Liang},~\bfnm{Yingyu}\binits{Y.}}
(\byear{2018}).
\btitle{Learning Overparameterized Neural Networks via Stochastic Gradient
  Descent on Structured Data}.
In \bbooktitle{Advances in Neural Information Processing Systems}
(\beditor{\bfnm{S.}\binits{S.}~\bsnm{Bengio}},
  \beditor{\bfnm{H.}\binits{H.}~\bsnm{Wallach}},
  \beditor{\bfnm{H.}\binits{H.}~\bsnm{Larochelle}},
  \beditor{\bfnm{K.}\binits{K.}~\bsnm{Grauman}},
  \beditor{\bfnm{N.}\binits{N.}~\bsnm{Cesa-Bianchi}} \AND
  \beditor{\bfnm{R.}\binits{R.}~\bsnm{Garnett}}, eds.)
\bvolume{31}.
\bpublisher{Curran Associates, Inc.}
\end{binproceedings}
\endbibitem

\bibitem{liu2019neural}
\begin{binproceedings}[author]
\bauthor{\bsnm{Liu},~\bfnm{Boyi}\binits{B.}},
  \bauthor{\bsnm{Cai},~\bfnm{Qi}\binits{Q.}},
  \bauthor{\bsnm{Yang},~\bfnm{Zhuoran}\binits{Z.}} \AND
  \bauthor{\bsnm{Wang},~\bfnm{Zhaoran}\binits{Z.}}
(\byear{2019}).
\btitle{Neural Trust Region/Proximal Policy Optimization Attains Globally
  Optimal Policy}.
In \bbooktitle{Advances in Neural Information Processing Systems}
(\beditor{\bfnm{H.}\binits{H.}~\bsnm{Wallach}},
  \beditor{\bfnm{H.}\binits{H.}~\bsnm{Larochelle}},
  \beditor{\bfnm{A.}\binits{A.}~\bsnm{Beygelzimer}},
  \beditor{\bfnm{F.}\binits{F.}~\bparticle{d\textquotesingle}
  \bsnm{Alch\'{e}-Buc}}, \beditor{\bfnm{E.}\binits{E.}~\bsnm{Fox}} \AND
  \beditor{\bfnm{R.}\binits{R.}~\bsnm{Garnett}}, eds.)
\bvolume{32}.
\bpublisher{Curran Associates, Inc.}
\end{binproceedings}
\endbibitem

\bibitem{ma2019proximally}
\begin{barticle}[author]
\bauthor{\bsnm{Ma},~\bfnm{Runchao}\binits{R.}},
  \bauthor{\bsnm{Lin},~\bfnm{Qihang}\binits{Q.}} \AND
  \bauthor{\bsnm{Yang},~\bfnm{Tianbao}\binits{T.}}
(\byear{2019}).
\btitle{Proximally constrained methods for weakly convex optimization with
  weakly convex constraints}.
\bjournal{arXiv preprint arXiv:1908.01871}.
\end{barticle}
\endbibitem

\bibitem{fard2016launch}
\begin{binproceedings}[author]
\bauthor{\bsnm{Milani~Fard},~\bfnm{Mahdi}\binits{M.}},
  \bauthor{\bsnm{Cormier},~\bfnm{Quentin}\binits{Q.}},
  \bauthor{\bsnm{Canini},~\bfnm{Kevin}\binits{K.}} \AND
  \bauthor{\bsnm{Gupta},~\bfnm{Maya}\binits{M.}}
(\byear{2016}).
\btitle{Launch and Iterate: Reducing Prediction Churn}.
In \bbooktitle{Advances in Neural Information Processing Systems}
(\beditor{\bfnm{D.}\binits{D.}~\bsnm{Lee}},
  \beditor{\bfnm{M.}\binits{M.}~\bsnm{Sugiyama}},
  \beditor{\bfnm{U.}\binits{U.}~\bsnm{Luxburg}},
  \beditor{\bfnm{I.}\binits{I.}~\bsnm{Guyon}} \AND
  \beditor{\bfnm{R.}\binits{R.}~\bsnm{Garnett}}, eds.)
\bvolume{29}.
\bpublisher{Curran Associates, Inc.}
\end{binproceedings}
\endbibitem

\bibitem{Na2021Adaptive}
\begin{barticle}[author]
\bauthor{\bsnm{Na},~\bfnm{Sen}\binits{S.}},
  \bauthor{\bsnm{Anitescu},~\bfnm{Mihai}\binits{M.}} \AND
  \bauthor{\bsnm{Kolar},~\bfnm{Mladen}\binits{M.}}
\btitle{An adaptive stochastic sequential quadratic programming with
  differentiable exact augmented lagrangians}.
\bdoi{10.1007/s10107-022-01846-z}
\end{barticle}
\endbibitem

\bibitem{Na2021Inequality}
\begin{barticle}[author]
\bauthor{\bsnm{Na},~\bfnm{Sen}\binits{S.}},
  \bauthor{\bsnm{Anitescu},~\bfnm{Mihai}\binits{M.}} \AND
  \bauthor{\bsnm{Kolar},~\bfnm{Mladen}\binits{M.}}
(\byear{2021}).
\btitle{Inequality Constrained Stochastic Nonlinear Optimization via Active-Set
  Sequential Quadratic Programming}.
\bjournal{Technical report}.
\end{barticle}
\endbibitem

\bibitem{narasimhan2019optimizing}
\begin{barticle}[author]
\bauthor{\bsnm{Narasimhan},~\bfnm{Harikrishna}\binits{H.}},
  \bauthor{\bsnm{Cotter},~\bfnm{Andrew}\binits{A.}} \AND
  \bauthor{\bsnm{Gupta},~\bfnm{Maya}\binits{M.}}
(\byear{2019}).
\btitle{Optimizing Generalized Rate Metrics through Game Equilibrium}.
\bjournal{arXiv preprint arXiv:1909.02939}.
\end{barticle}
\endbibitem

\bibitem{neyshabur2018pac}
\begin{binproceedings}[author]
\bauthor{\bsnm{Neyshabur},~\bfnm{Behnam}\binits{B.}},
  \bauthor{\bsnm{Bhojanapalli},~\bfnm{Srinadh}\binits{S.}} \AND
  \bauthor{\bsnm{Srebro},~\bfnm{Nathan}\binits{N.}}
(\byear{2018}).
\btitle{A {PAC}-Bayesian Approach to Spectrally-Normalized Margin Bounds for
  Neural Networks}.
In \bbooktitle{International Conference on Learning Representations}.
\end{binproceedings}
\endbibitem

\bibitem{neyshabur2018towards}
\begin{binproceedings}[author]
\bauthor{\bsnm{Neyshabur},~\bfnm{Behnam}\binits{B.}},
  \bauthor{\bsnm{Li},~\bfnm{Zhiyuan}\binits{Z.}},
  \bauthor{\bsnm{Bhojanapalli},~\bfnm{Srinadh}\binits{S.}},
  \bauthor{\bsnm{LeCun},~\bfnm{Yann}\binits{Y.}} \AND
  \bauthor{\bsnm{Srebro},~\bfnm{Nathan}\binits{N.}}
(\byear{2019}).
\btitle{The role of over-parametrization in generalization of neural networks}.
In \bbooktitle{International Conference on Learning Representations}.
\end{binproceedings}
\endbibitem

\bibitem{nocedal2006numerical}
\begin{bbook}[author]
\bauthor{\bsnm{Nocedal},~\bfnm{Jorge}\binits{J.}} \AND
  \bauthor{\bsnm{Wright},~\bfnm{Stephen}\binits{S.}}
(\byear{2006}).
\btitle{Numerical optimization}.
\bpublisher{Springer Science \& Business Media}.
\end{bbook}
\endbibitem

\bibitem{oneto2019general}
\begin{binproceedings}[author]
\bauthor{\bsnm{Oneto},~\bfnm{Luca}\binits{L.}},
  \bauthor{\bsnm{Donini},~\bfnm{Michele}\binits{M.}} \AND
  \bauthor{\bsnm{Pontil},~\bfnm{Massimiliano}\binits{M.}}
(\byear{2020}).
\btitle{General Fair Empirical Risk Minimization}.
In \bbooktitle{2020 International Joint Conference on Neural Networks (IJCNN)}
\bpages{1-8}.
\bdoi{10.1109/IJCNN48605.2020.9206819}
\end{binproceedings}
\endbibitem

\bibitem{oymak2019towards}
\begin{barticle}[author]
\bauthor{\bsnm{Oymak},~\bfnm{Samet}\binits{S.}} \AND
  \bauthor{\bsnm{Soltanolkotabi},~\bfnm{Mahdi}\binits{M.}}
(\byear{2020}).
\btitle{Toward Moderate Overparameterization: Global Convergence Guarantees for
  Training Shallow Neural Networks}.
\bjournal{IEEE Journal on Selected Areas in Information Theory}
\bvolume{1}
\bpages{84-105}.
\bdoi{10.1109/JSAIT.2020.2991332}
\end{barticle}
\endbibitem

\bibitem{salimans2016weight}
\begin{binproceedings}[author]
\bauthor{\bsnm{Salimans},~\bfnm{Tim}\binits{T.}} \AND
  \bauthor{\bsnm{Kingma},~\bfnm{Durk~P}\binits{D.~P.}}
(\byear{2016}).
\btitle{Weight Normalization: A Simple Reparameterization to Accelerate
  Training of Deep Neural Networks}.
In \bbooktitle{Advances in Neural Information Processing Systems}
(\beditor{\bfnm{D.}\binits{D.}~\bsnm{Lee}},
  \beditor{\bfnm{M.}\binits{M.}~\bsnm{Sugiyama}},
  \beditor{\bfnm{U.}\binits{U.}~\bsnm{Luxburg}},
  \beditor{\bfnm{I.}\binits{I.}~\bsnm{Guyon}} \AND
  \beditor{\bfnm{R.}\binits{R.}~\bsnm{Garnett}}, eds.)
\bvolume{29}.
\bpublisher{Curran Associates, Inc.}
\end{binproceedings}
\endbibitem

\bibitem{shalev2012online}
\begin{barticle}[author]
\bauthor{\bsnm{Shalev-Shwartz},~\bfnm{Shai}\binits{S.}}
(\byear{2012}).
\btitle{Online Learning and Online Convex Optimization}.
\bjournal{Foundations and Trends in Machine Learning}
\bvolume{4}
\bpages{107--194}.
\end{barticle}
\endbibitem

\bibitem{srebro2011universality}
\begin{binproceedings}[author]
\bauthor{\bsnm{Srebro},~\bfnm{Nati}\binits{N.}},
  \bauthor{\bsnm{Sridharan},~\bfnm{Karthik}\binits{K.}} \AND
  \bauthor{\bsnm{Tewari},~\bfnm{Ambuj}\binits{A.}}
(\byear{2011}).
\btitle{On the Universality of Online Mirror Descent}.
In \bbooktitle{Advances in Neural Information Processing Systems}
(\beditor{\bfnm{J.}\binits{J.}~\bsnm{Shawe-Taylor}},
  \beditor{\bfnm{R.}\binits{R.}~\bsnm{Zemel}},
  \beditor{\bfnm{P.}\binits{P.}~\bsnm{Bartlett}},
  \beditor{\bfnm{F.}\binits{F.}~\bsnm{Pereira}} \AND
  \beditor{\bfnm{K.~Q.}\binits{K.~Q.}~\bsnm{Weinberger}}, eds.)
\bvolume{24}.
\bpublisher{Curran Associates, Inc.}
\end{binproceedings}
\endbibitem

\bibitem{t2020personalized}
\begin{binproceedings}[author]
\bauthor{\bsnm{T.~Dinh},~\bfnm{Canh}\binits{C.}},
  \bauthor{\bsnm{Tran},~\bfnm{Nguyen}\binits{N.}} \AND
  \bauthor{\bsnm{Nguyen},~\bfnm{Josh}\binits{J.}}
(\byear{2020}).
\btitle{Personalized Federated Learning with Moreau Envelopes}.
In \bbooktitle{Advances in Neural Information Processing Systems}
(\beditor{\bfnm{H.}\binits{H.}~\bsnm{Larochelle}},
  \beditor{\bfnm{M.}\binits{M.}~\bsnm{Ranzato}},
  \beditor{\bfnm{R.}\binits{R.}~\bsnm{Hadsell}},
  \beditor{\bfnm{M.~F.}\binits{M.~F.}~\bsnm{Balcan}} \AND
  \beditor{\bfnm{H.}\binits{H.}~\bsnm{Lin}}, eds.)
\bvolume{33}
\bpages{21394--21405}.
\bpublisher{Curran Associates, Inc.}
\end{binproceedings}
\endbibitem

\bibitem{zafar2017fairness}
\begin{barticle}[author]
\bauthor{\bsnm{Zafar},~\bfnm{Muhammad~Bilal}\binits{M.~B.}},
  \bauthor{\bsnm{Valera},~\bfnm{Isabel}\binits{I.}},
  \bauthor{\bsnm{Gomez-Rodriguez},~\bfnm{Manuel}\binits{M.}} \AND
  \bauthor{\bsnm{Gummadi},~\bfnm{Krishna~P.}\binits{K.~P.}}
(\byear{2019}).
\btitle{Fairness Constraints: A Flexible Approach for Fair Classification}.
\bjournal{Journal of Machine Learning Research}
\bvolume{20}
\bpages{1--42}.
\end{barticle}
\endbibitem

\bibitem{zhou2018fenchel}
\begin{barticle}[author]
\bauthor{\bsnm{Zhou},~\bfnm{Xingyu}\binits{X.}}
(\byear{2018}).
\btitle{On the fenchel duality between strong convexity and lipschitz
  continuous gradient}.
\bjournal{arXiv preprint arXiv:1803.06573}.
\end{barticle}
\endbibitem

\bibitem{zou2018stochastic}
\begin{barticle}[author]
\bauthor{\bsnm{Zou},~\bfnm{Difan}\binits{D.}},
  \bauthor{\bsnm{Cao},~\bfnm{Yuan}\binits{Y.}},
  \bauthor{\bsnm{Zhou},~\bfnm{Dongruo}\binits{D.}} \AND
  \bauthor{\bsnm{Gu},~\bfnm{Quanquan}\binits{Q.}}
(\byear{2019}).
\btitle{Gradient descent optimizes over-parameterized deep {ReLU} networks}.
\bjournal{Machine Learning}
\bvolume{109}
\bpages{467--492}.
\bdoi{10.1007/s10994-019-05839-6}
\end{barticle}
\endbibitem

\end{thebibliography}


\end{document}